\documentclass[journal]{IEEEtran}
\usepackage{times}
\usepackage{epsfig}
\usepackage{graphicx}
\usepackage{amsmath}
\usepackage{amssymb}
\usepackage[font=small]{caption}
\usepackage{subcaption} 
\usepackage{xcolor}
\usepackage{caption}
\usepackage{amsthm}
\usepackage{mathrsfs} 
\usepackage{enumerate}
\usepackage{multirow}
\usepackage{tabularx}
\usepackage{float}
\usepackage{color}
\usepackage{times}
\usepackage{cite}
\usepackage[hidelinks]{hyperref}
\usepackage{comment}

\def\R{\mathbb{R}}
\def\U{\mathbb{U}}
\def\V{\mathbb{V}}

\def\G{\mathcal{G}_R}
\def\GO{\mathcal{G}}
\def\S{\mathbb{S}}
\def\RR{\mathscr{R}}

\def\Diff{\mathcal{T}}
\newcommand{\Sk}{\widehat{\mathbb{S}}^{(k)}}
\newcommand{\Vk}{\widehat{\mathbb{V}}^{(k)}}

\newtheorem{theorem}{Theorem}[section]

\newtheorem{lemma}[theorem]{Lemma}

\newtheorem{definition}[theorem]{Definition}
\newtheorem{remark}[theorem]{Remark}

\newtheorem{example}{Example}

\usepackage{enumitem}
\usepackage{makecell}

\begin{document}

\title{Radon Cumulative Distribution Transform Subspace Modeling for Image Classification}

\author{Mohammad Shifat-E-Rabbi*, Xuwang Yin$^\dagger$, Abu Hasnat Mohammad Rubaiyat$^\dagger$, Shiying Li, Soheil Kolouri, Akram Aldroubi, Jonathan M. Nichols, and Gustavo K. Rohde
\thanks{*M. Shifat-E-Rabbi and S. Li are with the Department of Biomedical Engineering, University of Virginia, Charlottesville, VA 22908, USA (e-mail: *mr2kz@virginia.edu, sl8jx@virginia.edu).}
\thanks{X. Yin and A. H. M. Rubaiyat are with the Department of Electrical and Computer Engineering, University of Virginia, Charlottesville, VA 22904, USA (e-mail: xy4cm@virginia.edu, ar3fx@virginia.edu).}%
\thanks{S. Kolouri is with the HRL Laboratories, LLC, Malibu, CA 90265, USA (e-mail: skolouri@hrl.com).}%
\thanks{A. Aldroubi is with the Department of Mathematics, Vanderbilt University, Nashville, TN 37212, USA (e-mail: akram.aldroubi@vanderbilt.edu).}%
\thanks{J. M. Nichols is with the U.S. Naval Research Laboratory, Washington, DC 20375, USA (e-mail: jonathan.nichols@nrl.navy.mil).}%
\thanks{ G. K. Rohde is with the Department of Biomedical Engineering and the Department of Electrical and Computer Engineering, University of Virginia, Charlottesville, VA 22908, USA (e-mail: gustavo@virginia.edu).}%
\thanks{* indicates corresponding author, $\dagger$ indicates equal contribution.}
\thanks{ }
\thanks{\textcopyright{ \it{Journal of Mathematical Imaging and Vision}} (2021) 63:1185–1203. Permission from the journal must be obtained for all uses.}
}

\maketitle

\begin{abstract}
We present a new supervised image classification method applicable to a broad class of image deformation models. The method makes use of the previously described Radon Cumulative Distribution Transform (R-CDT) for image data, whose mathematical properties are exploited to express the image data in a form that is more suitable for machine learning. While certain operations such as translation, scaling, and higher-order transformations are challenging to model in native image space, we show the R-CDT can capture some of these variations and thus render the associated image classification problems easier to solve. The method -- utilizing a nearest-subspace algorithm in R-CDT space -- is simple to implement, non-iterative, has no hyper-parameters to tune, is computationally efficient, label efficient, and provides competitive accuracies to state-of-the-art neural networks for many types of classification problems. In addition to the test accuracy performances, we show improvements (with respect to neural network-based methods) in terms of computational efficiency (it can be implemented without the use of GPUs), number of training samples needed for training, as well as out-of-distribution generalization. The Python code for reproducing our results is available at \cite{software}. 


\end{abstract}

\IEEEpeerreviewmaketitle

\section{Introduction}

\IEEEPARstart{I}{mage} classification refers to the process of automatic image class prediction based on the numerical content of their corresponding pixel values. Automated image classification methods have been used to detect cancer from microscopy images of tumor specimens \cite{sertel2009computer}\cite{basu2014detecting}, detect and quantify atrophy from magnetic resonance images of the human brain \cite{kundu2018discovery}\cite{schulz2010visualization},  identify and authenticate a person from cell phone camera images \cite{hadid2007face}, and numerous other applications in computer vision, medical imaging, automated driving and others. 

While many methods for automated image classification have been developed, those based on supervised learning have attracted most of the attention given that {\it a priori} knowledge of the image data usually leads to more accurate classifiers than the unsupervised alternatives. In supervised learning, a set of labeled example images (known as training data) is utilized to estimate the value of parameters of a mathematical model to be used for classification. Given an unknown test image, the goal of the classification method is to automatically assign the label or class of that image. 

An extensive set of supervised learning-based classification algorithms have been proposed in the past (see \cite{shifat2020cell,rawat2017deep,lu2007survey} for a few reviews on the subject). Two broad categories of these algorithms are: 1) learning of classifiers on hand-engineered features and 2) end-to-end learning of features and classifiers, e.g., hierarchical neural networks. Certainly, many algorithms exist that may fit into more than one category, while other algorithms may not fit into any. However, for the purposes of our discussion we focus on these two broad categories.

Image classification methods based on {\it{hand-engineered features}}, perhaps the first to arise \cite{prewitt1966analysis}, generally work in a two step process: step one being the extraction of numerical features that model the pixel intensities, and step two being the application of statistical classification methods to those features. A large number of numerical features have been engineered in the past to represent the information from a given image, including Haralick features, Gabor features, shape features \cite{orlov2008wnd}\cite{ponomarev2014ana}, and numerous others \cite{shifat2020cell}. These are then combined with many different multivariate regression-based classification methods including linear discriminant analysis \cite{bandos2009classification} \cite{muldoon2010evaluation}, support vector machines \cite{zhang2007local} \cite{perronnin2010improving}, random forests \cite{bosch2007image} \cite{du2015random}, as well as their kernel versions.

\begin{figure*}[!hbt]
    \centering
    \includegraphics[width=18cm]{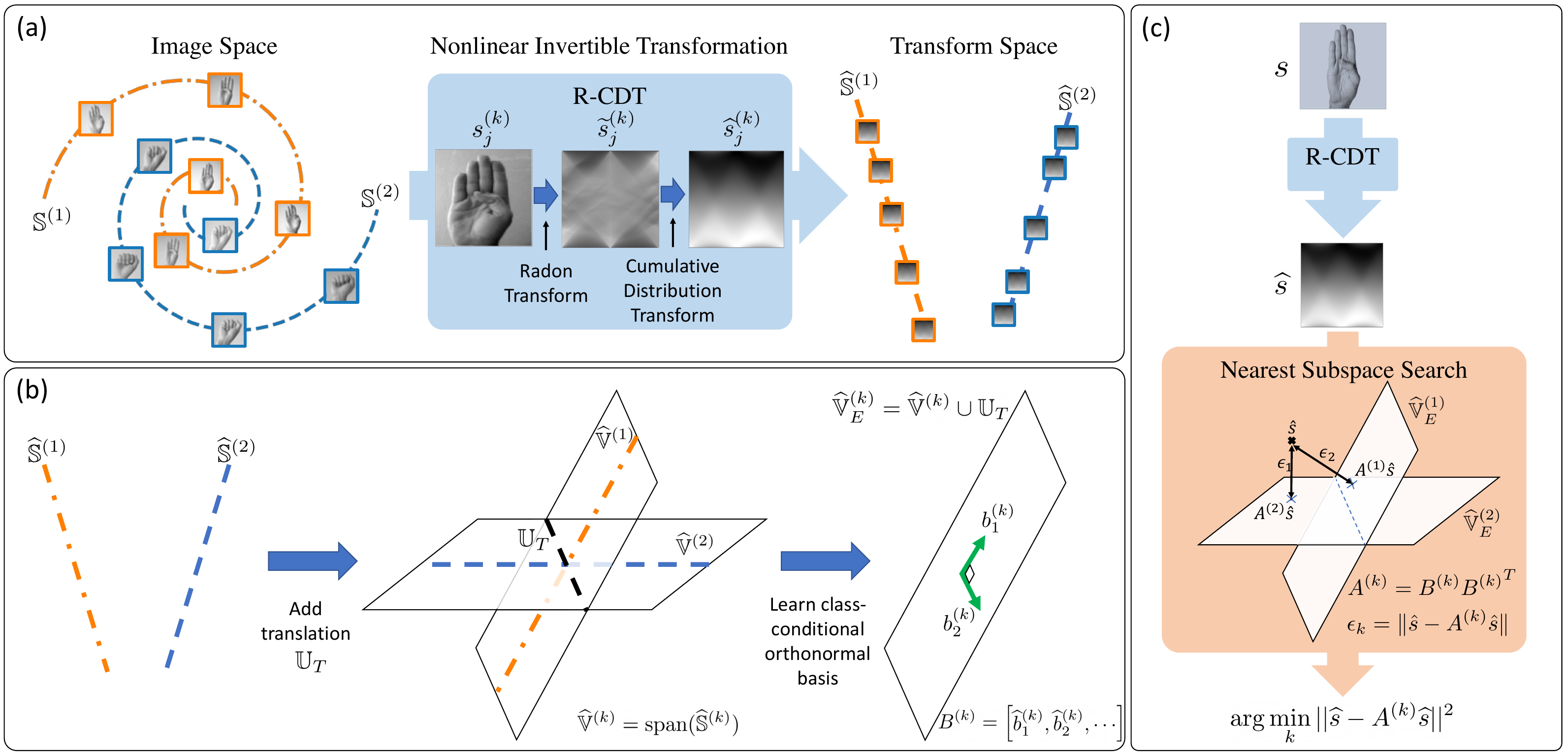}
    \caption{System diagram outlining the proposed Radon cumulative distribution transform subspace modeling technique for image classification. (a) R-CDT - a nonlinear, invertible transformation: The R-CDT transform simplifies the data space; (b) Generative modeling - subspace learning: the simplified data spaces can be modeled as linear subspaces; (c) Classification pipeline: the classification method consists of the R-CDT transform followed by a nearest subspace search in the R-CDT space.}
    \label{figres00}
    \vspace{-1em}
\end{figure*}

Methods based on hierarchical neural networks \cite{lecun2015deep}, such as {\it{convolutional neural networks}} (CNNs) \cite{rawat2017deep} \cite{shin2016deep}, have been widely studied recently given they have achieved top performance in certain classification tasks \cite{lecun2015deep} \cite{szegedy2016rethinking} \cite{szegedy2015going}. In contrast to hand-engineered features, CNNs typically combine both feature extraction and classification methods within one consistent framework, i.e.,  end-to-end learning. 
The unprecedented performance of the deep neural networks on a wide variety of tasks has made them quintessential to modern supervised learning on images. These methods, however, are: 1) computationally expensive, often requiring graphic processing units (GPUs) to train and deploy, 2) data-hungry, requiring thousands of labeled images per class, and 3) often vulnerable against out-of-distribution samples, e.g., adversarial attacks. 

A less commonly used alternative is to model an observed image as the deformation of another image.  To this end, image ``morphing'' models have been designed to capture, for example, translation and scalings among two or more images \cite{wolberg1998morph}. Recently, a new class of general ``transport'' models have been developed which describe an image as a smooth, nonlinear, invertible transformation of a reference image \cite{kolouri2016radon,kolouri2017optimal}.  The estimation of such models from observed imagery is greatly facilitated by the R-CDT, a newly developed image transform \cite{kolouri2016radon}. Unlike most numerical feature based methods described above, this operation is invertible as the R-CDT is an invertible image transform and thus the R-CDT can be viewed as a mathematical image representation method. The R-CDT developed in \cite{kolouri2016radon} has connections to optimal transport theory \cite{kolouri2017optimal}\cite{villani2008optimal}. In particular, the R-CDT can be interpreted as the application of the linear optimal transportation concept \cite{wang2013linear} to the sliced Wasserstein distance \cite{kolouri2016sliced}. It can also be interpreted as a nonlinear kernel \cite{kolouri2016sliced}. 

The R-CDT, and linear optimal transport models \cite{wang2013linear}, have been applied to image classification before in combination with linear classifiers such as Fisher discriminant analysis, support vector machines, and their respective kernel techniques \cite{kolouri2016radon} \cite{kolouri2016sliced}\cite{park2018multiplexing}. While successful in certain applications \cite{park2018multiplexing}, this approach of classification using the R-CDT has failed to produce the state of the art classification results in certain other applications (see Figure~3 from \cite{shifat2020cell}). 

In this work we improve upon past performance and develop a new R-CDT approach to supervised image classification.  We first highlight the many useful mathematical properties of the R-CDT and then show the implications of these properties for the classification problem. We then leverage these properties to propose a new R-CDT classifier and demonstrate the performance of that classifier in numerous applications. Figure~\ref{figres00} shows a system diagram outlining the main computational modeling steps in the proposed method. Our specific contributions are therefore as follows:

\subsection*{Our contributions}
\begin{itemize}
    \item We propose a classification algorithm which offers competitive accuracy performance in comparison with deep learning based methods.
    \item The algorithm requires very few labeled data to train and outperforms deep learning by a large margin in limited training sample size setting. The proposed method is also exceptionally cheap in terms of computation; up to 10,000 times savings in computational complexity can be attained, as compared with the deep-learning-based methods, to achieve the same test accuracy.
    
    \item A particular compelling property of the proposed method is the robustness under out-of-distribution setups, meaning, our model generalizes to data that were previously unobserved. This property is a direct result of the Lemmas derived in section IV which speak to the convexity and separability of the data in transform space. 
    
    \item We arrive at the proposed algorithm by expanding and improving upon the R-CDT-based image classification technique. Utilizing the properties of the CDT \cite{park2018cumulative} and R-CDT \cite{kolouri2016radon} we propose that each class can be modeled as a convex subspace in R-CDT domain. We mathematically show that the data space in R-CDT domain do not intersect with the subspace corresponding to a different class. In light of these properties, the algorithm implements a nearest subspace search in R-CDT domain to classify the test images.
\end{itemize}

The paper is organized as follows: Section \ref{prelim} presents a preliminary overview of a family of transport-based nonlinear transforms: the CDT for 1D signals, and the R-CDT for 2D images. The classification problem is stated in Section \ref{plm_state} with the proposed solution in Section \ref{plm_soln}. Descriptions of the experimental setup and the datasets used are available in Section~\ref{mnm}. Experimental results are presented in Section \ref{result} with the discussion of the results in Section \ref{discuss}. Finally, Section \ref{conclude} offers concluding remarks.

\begin{table}[]
\centering
\caption{\normalsize{Description of symbols}}
\label{table:symbols}
\begin{tabular}{ll}
\hline
Symbols                & Description    \\ \hline
$s(x)~/~s(\mathbf{x})$ & Signal / image \\
$\Omega_s$& Domain of $s$\\
$\widetilde{s}(t,\theta)$ & Radon transform of $s$\\
$\widehat{s}(x)~/~\widehat{s}(t,\theta)$ & CDT / R-CDT transform of $s$\\
$\mathscr{R}(\cdot)~/~\mathscr{R}^{-1}(\cdot)$ & Forward / inverse Radon transform operation\\
$g(x)$ & Strictly increasing and differentiable function\\
$g^{\theta}(t)$ & Strictly increasing and differentiable \\& function, indexed by an angle $\theta$ \\
$s\circ g$&$s(g(x))$: composition of $s(x)$ with $g(x)$\\
$\widetilde{s}\circ g^\theta$&$\widetilde{s}(g^\theta(t),\theta)$: composition of $\widetilde{s}(t,\theta)$ with $g^{\theta}(t)$\\& along the $t$ dimension of $\widetilde{s}(t,\theta)$\\
$\GO$ &  Set of increasing diffeomorphisms $g(x)$ \\
$\G$ & Set of increasing diffeomorphisms $g^{\theta}(t)$ \\&parameterized by $\theta$ with $\theta \in [0,\pi]$\\
$\Diff$& Set of all possible increasing diffeomorphisms \\&from $\R$ to $\R$\\

\hline
\end{tabular}
\vspace{-1em}
\end{table}

\section{Preliminaries}
\label{prelim}
\subsection{Notation}
\label{prelim_a}

Throughout the manuscript, we deal with signals $s$ assuming these to be square integrable in their respective domains. That is, we assume that $\int_{\Omega_s} |s(x)|^2 dx < \infty$, where $\Omega_s\subseteq\R$ is the domain over which $s$ is defined. In addition, we at times make use of the common notation: $\| s \|^2 = <s, s> = \int_{\Omega_s} s(x)^* s(x) dx = \int_{\Omega_s} |s(x)|^2 dx $, where $<\cdot, \cdot>$ is the inner product. Signals are assumed to be real, so the complex conjugate $^*$ does not play a role.
We will apply the same notation for functions whose input argument is two dimensional, i.e. images. Let $\mathbf{x} \in \Omega_s \subseteq \R^2$. A 2D continuous function representing the continuous image is denoted $s(\mathbf{x}), \mathbf{x} \in \Omega_s$. Signals or images are denoted $s^{(k)}$ when the class information is available, where the superscript $(k)$ represents the class label.


Below we will also make use of one dimensional (1D) increasing diffeomorphisms (one to one mapping functions), which are denoted as $g(x)$ for signals and $g^\theta(t)$ when they need to be parameterized by an angle $\theta$. The set of all possible increasing diffeomorphisms from $\R$ to $\R$ will be denoted as $\Diff$. 
Finally, at times we also utilize the `$\circ$' operator to denote composition. 
A summary of the symbols and notation used can be found in Table~\ref{table:symbols}.

\subsection{The Cumulative Distribution Transform (CDT)}
\label{prelim_b}
The CDT \cite{park2018cumulative} is an invertible nonlinear 1D signal transform from the space of smooth probability densities to the space of diffeomorphisms. The CDT morphs a given input signal, defined as a probability density function (PDF), into another PDF in such a way that the Wasserstein distance between them is minimized. More formally, let $s(x),x\in\Omega_s$ and $r(x),x\in\Omega_r$ define a given signal and a reference signal, respectively, which we consider to be appropriately normalized such that $s>0,r>0$, and $\int_{\Omega_s} s(x)dx=\int_{\Omega_r} r(x)dx=1$. The forward CDT transform\footnote{We are using a slightly different definition of the CDT than in \cite{park2018cumulative}. The properties of the CDT outlined here hold in both definitions.} of $s(x)$ with respect to $r(x)$ is given by the strictly increasing function $\widehat{s}(x)$ that satisfies
\begin{align}
    \int_{-\infty}^{\widehat{s}(x)}s(u)du=\int_{-\infty}^{x}r(u)du\nonumber
\end{align}
As described in detail in \cite{park2018cumulative}, the CDT is a nonlinear and invertible operation, with the inverse being
\begin{align}
s(x)=\frac{d\widehat{s}^{-1}(x)}{dx}r\left(\widehat{s}^{-1}(x)\right), ~\mbox{and}~\widehat{s}^{-1}(\widehat{s}(x)) = x\nonumber
\end{align}

Moreover, like the Fourier transform \cite{bracewell1986fourier} for example, the CDT has a number of properties which will help us render signal and image classification problems easier to solve.\\\\
{\bf{Property II-B.1}}~{\it{(Composition)}}: Let $s(x)$ denote a normalized signal and let $\widehat{s}(x)$ be the CDT of $s(x)$. The CDT of $s_g=g's\circ g$ is given by 
\begin{align}
    \widehat{s}_g=g^{-1}\circ\widehat{s}
\end{align}
Here, $g\in\Diff$ is an invertible and differentiable function (diffeomorphism), $g^\prime=dg(x)/dx$, and `$\circ$' denotes the composition operator with $s\circ g=s(g(x))$. For a proof, see Appendix~A in supplementary materials.

\begin{figure}
    \centering
    \includegraphics[width=8.85cm]{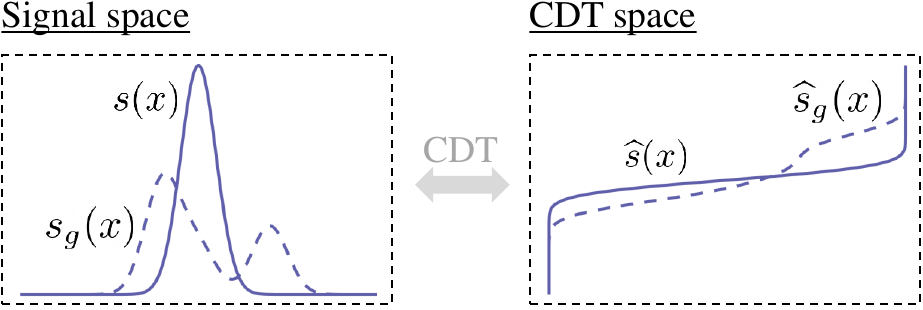}
    \caption{The cumulative distribution transform (CDT) of a signal (probability density function). Note that the CDT of an altered (transported) signal $s_g(x)$ (see text for definition) is related to the transform of $s$. In short, the CDT renders displacements into amplitude modulations in transform space.}
    \label{fig:cdt_comp}
    \vspace{-1em}
\end{figure}

The CDT composition property implies that, variations in a signal caused by applying $g(x)$ to the independent variable will change only the dependent variable in CDT space. This property is illustrated in Figure~\ref{fig:cdt_comp} where variations along both independent and dependent axis directions in original signal space become changes solely along the dependent axis in CDT space).\\\\
{\bf{Property II-B.2}} {\it{(Embedding)}}: 
CDT induces an isometric embedding between the space of 1D signals  with the 2-Wasserstein metric and the space of their CDT transforms with a weighted-Euclidean metric \cite{kolouri2016radon}\cite{park2018cumulative}, i.e.,
\begin{equation}\label{eq: cdtembedding}
	 W_2^2(s_1,s_2) =\left|\left|\left(\widehat{s}_1-\widehat{s}_2\right)\sqrt{r}\right|\right|_{L^2(\Omega_r)}^2,
\end{equation}
for all signals $s_1, s_2$. That is to say, if we wish to use the Wasserstein distance as a measure of similarity between $s_1,~s_2$, we can compute it as simply a weighted Euclidean norm in CDT space. For a proof, see Appendix~C in supplementary materials.

The property above naturally links the CDT and Wasserstein distances for PDFs. Wasserstein \cite{villani2008optimal} distances are linked to optimal transport and have been used in a variety of applications in signal and image processing and machine learning (see \cite{kolouri2017optimal} for a recent review). 


\subsection{The Radon transform}
\label{prelim_c}
The Radon transform of an image $s(\mathbf{x}),\mathbf{x}\in\Omega_s\subset \R^2$, which we denote by $\widetilde{s}=\RR(s)$, is defined as
\begin{eqnarray}
\widetilde{s}(t,\theta)&=&\int_{\Omega_s}s(\mathbf{x})\delta(t-\mathbf{x}\cdot \mathbf{\xi}_\theta)d\mathbf{x}
\end{eqnarray}
Here, $t$ is the perpendicular distance of a line from the origin and $\xi_\theta = [\cos(\theta),\sin(\theta)]^T$, where $\theta$ is the angle over which the projection is taken. 

Furthermore, using the Fourier Slice Theorem \cite{quinto2006introduction}\cite{natterer2001mathematics}, the inverse Radon transform $s=\RR^{-1}(\widetilde{s})$ is defined as
\begin{eqnarray}
s(\mathbf{x})&=&\int_0^\pi\int_{-\infty}^{\infty}\widetilde{s}(\mathbf{x}\cdot\xi_\theta-\tau,\theta)w(\tau)d\tau d\theta,
\end{eqnarray}
where $w$ is the ramp filter (i.e.,$(\mathscr Fw) (\xi) = |\xi|, \forall \xi$ ) and $\mathscr{F}$ is the Fourier transform.\\\\
{\bf{Property II-C.1}} ({\it{Intensity equality}}): Note that
\begin{eqnarray}
\int_{\Omega_s} s(\mathbf{x})d\mathbf{x} =\int_{-\infty}^{\infty}\widetilde{s}(t,\theta)dt, \;\;\;\;\; \forall\theta\in[0,\pi]
\label{eq:Rtheta}
\end{eqnarray}
which implies that $\int_{-\infty}^{\infty}\widetilde{s}(t,\theta_i)dt=\int_{-\infty}^{\infty}\widetilde{s}(t,\theta_j)dt$ for any two choices $\theta_i,\theta_j\in[0,\pi]$.

\subsection{Radon Cumulative Distribution Transform (R-CDT)}
\label{prelim_d}
The CDT framework was extended for 2D patterns (images as normalized density functions) through the sliced-Wasserstein distance in \cite{kolouri2016radon}, and was denoted as R-CDT. The main idea behind the R-CDT is to first obtain a family of one dimensional representations of a two dimensional probability measure (e.g., an image) through the Radon transform and then apply the CDT over the $t$ dimension in Radon transform space. More formally, let $s(\mathbf{x})$ and $r(\mathbf{x})$ define a given image and a reference image, respectively, which we consider to be appropriately normalized. The forward R-CDT of $s(\mathbf{x})$ with respect to $r(\mathbf{x})$ is given by the measure preserving function $\widehat{s}(t,\theta)$ that satisfies
\begin{align}\label{eq:rcdt}
    \int_{-\infty}^{\widehat{s}(t,\theta)}\widetilde{s}(u,\theta)du=\int_{-\infty}^{t}\widetilde{r}(u,\theta)du,~~~\forall\theta\in[0,\pi]
\end{align}

\begin{figure}
    \centering
    \includegraphics[width=8.85cm]{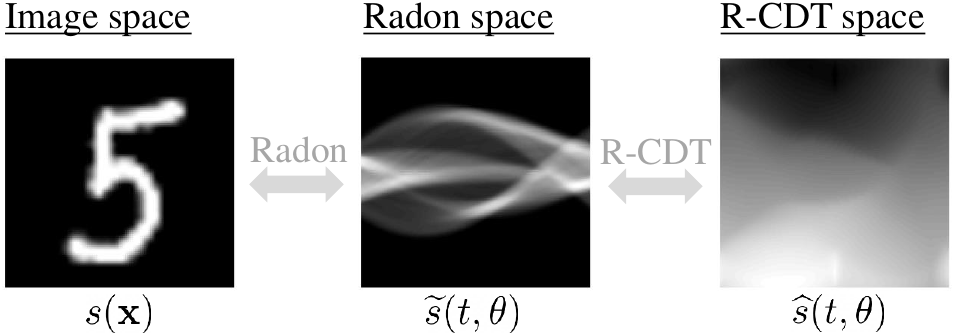}
    \caption{The process of calculating the Radon cumulative distribution transform (R-CDT) of an image $s(\mathbf{x})$ (defined as a 2-dimensional probability density function). The first step is to apply the Radon transform on $s(\mathbf{x})$ to obtain $\widetilde{s}(t,\theta)$. The R-CDT $\widehat{s}(t,\theta)$ is then obtained by applying the CDT over the $t$ dimension of $\widetilde{s}(t,\theta),~\forall\theta$.} 
    \label{fig:rcdt_process}
    \vspace{-1em}
\end{figure}

As in the case of the CDT, a transformed signal in R-CDT space can be recovered via the following inverse formula  \cite{kolouri2016radon}, 
\begin{align} 
s(\mathbf{x})=\RR^{-1}\left(\frac{\partial \widehat{s}^{-1}(t,\theta)}{\partial t}\widetilde{r}\left(\widehat{s}^{-1}(t,\theta),\theta\right)\right)\nonumber
\end{align}
The process of calculating the R-CDT transform is shown in Figure~\ref{fig:rcdt_process}. As with the CDT, the R-CDT has a couple of properties outlined below which will be of interest when classifying images.\\\\
{\bf{Property II-D.1}} {\it{(Composition)}}: Let $s(\mathbf{x})$ denotes an appropriately normalized image and let $\widetilde{s}(t,\theta)$ and $\widehat{s}(t,\theta)$ be the Radon transform and the R-CDT transform of $s(\mathbf{x})$, respectively. The R-CDT transform of $s_{g^\theta}=\mathscr{R}^{-1}\left(\left({g^\theta}\right)^\prime\widetilde{s}\circ {g^\theta}\right)$ is given by
\begin{align}
    \widehat{s}_{g^\theta}=({g^\theta})^{-1}\circ\widehat{s},
\end{align}
where $\left(g^\theta\right)^\prime=dg^\theta(t)/dt$, $\widetilde{s}\circ {g^\theta}:=\widetilde{s}(g^\theta(t),\theta)$, and $({g^\theta})^{-1}\circ\widehat{s}=({g^\theta})^{-1}(\widehat{s}(t,\theta))$.
Here for a fixed $\theta$, $g^\theta$ can be thought of an increasing and differentiable function with respect to $t$. The above equation hence follows from the composition property for 1D CDT. For a proof, see Appendix~B in supplementary materials.

The R-CDT composition property implies that, variations along both independent and dependent axis directions in an image, caused by applying $g^\theta(t)$ to the independent $t$ variable of its Radon transform, become changes solely along the dependent variable in R-CDT space.\\\\
{\bf{Property II-D.2}} {\it{(Embedding)}}:
R-CDT induces an isometric embedding between the space of images with  sliced-Wasserstein metric and the space of their R-CDT transforms with a weighted-Euclidean metric, i.e.,
\begin{equation}\label{eq: rcdtembedding}
	SW_2^2(s_1,s_2) = \left|\left|\left(\widehat{s}_1-\widehat{s}_2\right)\sqrt{\widetilde{r}}\right|\right|_{L^2(\Omega_{\widetilde r})}^2
\end{equation}
for all images $s_1$ and $s_2$. For a proof, see Appendix~D in supplementary materials. 

As the case with the 1D CDT shown above, the property above naturally links the R-CDT and sliced Wasserstein distances for PDFs and affords us a simple means of computing similarity among images \cite{kolouri2016radon}. We remark that throughout this manuscript we use the notation $\widehat s$ for both CDT or R-CDT transforms of a signal or image $s$ with respect to a fixed reference signal or image $r$, if a reference is not specified.



\section{Generative Model and Problem Statement}
\label{plm_state}
Using the notation established above we are ready to discuss a generative model-based problem statement for the type of classification problems we discuss in this paper. We begin by noting that in many applications we are concerned with classifying image or signal patterns that are instances of a certain prototype (or template) observed under some often unknown deformation pattern.
Consider the problem of classifying handwritten digits (e.g. the MNIST dataset \cite{lecun1998gradient}). A good model for each class in such a dataset is to assume that each observed digit image can be thought of as being an instance of a template (or templates) observed under some (unknown) deformation or similar variation or confound. For example, a generative model for the set of images of the digit 1 could be a fixed pattern for the digit 1, but observed under different translations -- the digit can be positioned randomly within the field of view of the image. Alternatively, the digit could also be observed with different sizes, or slight deformations. The generative models stated below for 1D and 2D formalize these statements.

\begin{example}[1D generative model with translation]
Consider a 1D signal pattern denoted as $\varphi^{(k)}$ (the superscript $(k)$ here denotes the class in a classification problem), observed under a random translation parameter $\mu$. In this case, we can mathematically represent the situation by defining the set of all possible functions $g(x) = x - \mu$, with $\mu$ being a random variable whose distribution is typically unknown. A random observation (randomly translated) pattern can be written mathematically as $g^\prime(x)\varphi^{(k)}(g(x))$. Note that in this case $g^\prime(x) = 1$, and thus the generative model simply amounts to random translation of a template pattern. Figure~\ref{fig:generative_model_example} depicts this situation.
\end{example}

The example above (summarized in Figure~\ref{fig:generative_model_example}) can be expressed in more general form. Let $\GO \subset \Diff$ denotes a set of 1D spatial transformations of a specific kind (e.g. the set of affine transformations). We then use these transformations to provide a more general definition for a mass (signal intensity) preserving generative data model.
\begin{figure}
    \centering
    \includegraphics[width=6cm]{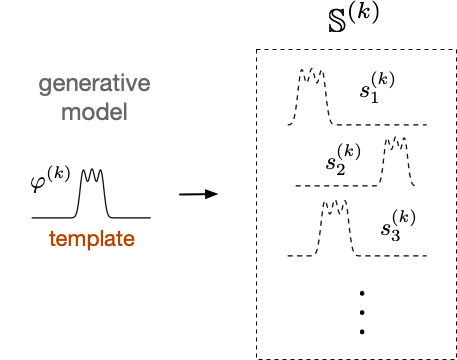}
    \caption{Generative model example. A signal generative model can be constructed by applying randomly drawn confounding spatial transformations, in this case translation $(g(x)=x-\mu)$, to a template pattern from class $(k)$, denoted here as $\varphi^{(k)}$. The notation $s_j^{(k)}$ here is meant to denote the $j^{\mbox{th}}$ signal from the $k^{\mbox{th}}$ class.}
    \label{fig:generative_model_example}
    \vspace{-1em}
\end{figure}

\begin{definition}[1D generative model]
\label{def:gen_model_1d}
    Let $\GO \subset \Diff$. The 1D mass (signal intensity) preserving generative model for the $k^{\mbox{th}}$ class is defined to be the set 
\begin{eqnarray}
\label{eq:1dgenerative_model}\mathbb{S}^{(k)}&=&\{s_j^{(k)}|s_j^{(k)}=g_j'\varphi^{(k)}\circ g_j, \forall g_j\in\GO \}.
\end{eqnarray}
\end{definition}
The notation $s_j^{(k)}$ here is meant to denote the $j^{\text{th}}$ signal from the $k^{\text{th}}$ class. The derivative term $g^\prime_j$ preserves the normalization of signals. This extension allows us to define and discuss problems where the confound goes beyond a simple translation model.\\

With the definition of the 2-Dimensional Radon transform from section \ref{prelim_c}, we are now ready to define the 2-dimensional definition of the generative data model we use throughout the paper:


\begin{definition}[2D generative model]
\label{def:gen_model_2d}
    Let $\G\subset\Diff$ be our set of confounds. The 2D mass (image intensity) preserving generative model for the $k^{\mbox{th}}$ class is defined to be the set 
\begin{eqnarray}
\mathbb{S}^{(k)}=\left\{s_j^{(k)}|s_j^{(k)}=\RR^{-1}\left(\left({g_j^\theta}\right)^\prime\widetilde{\varphi}^{(k)}\circ g^\theta_j\right), \forall g^\theta_j\in\G \right\}.
\label{eq:2dgenerative_model}
\end{eqnarray}
\end{definition}

\begin{figure}
    \centering
    \includegraphics[width=8.9cm]{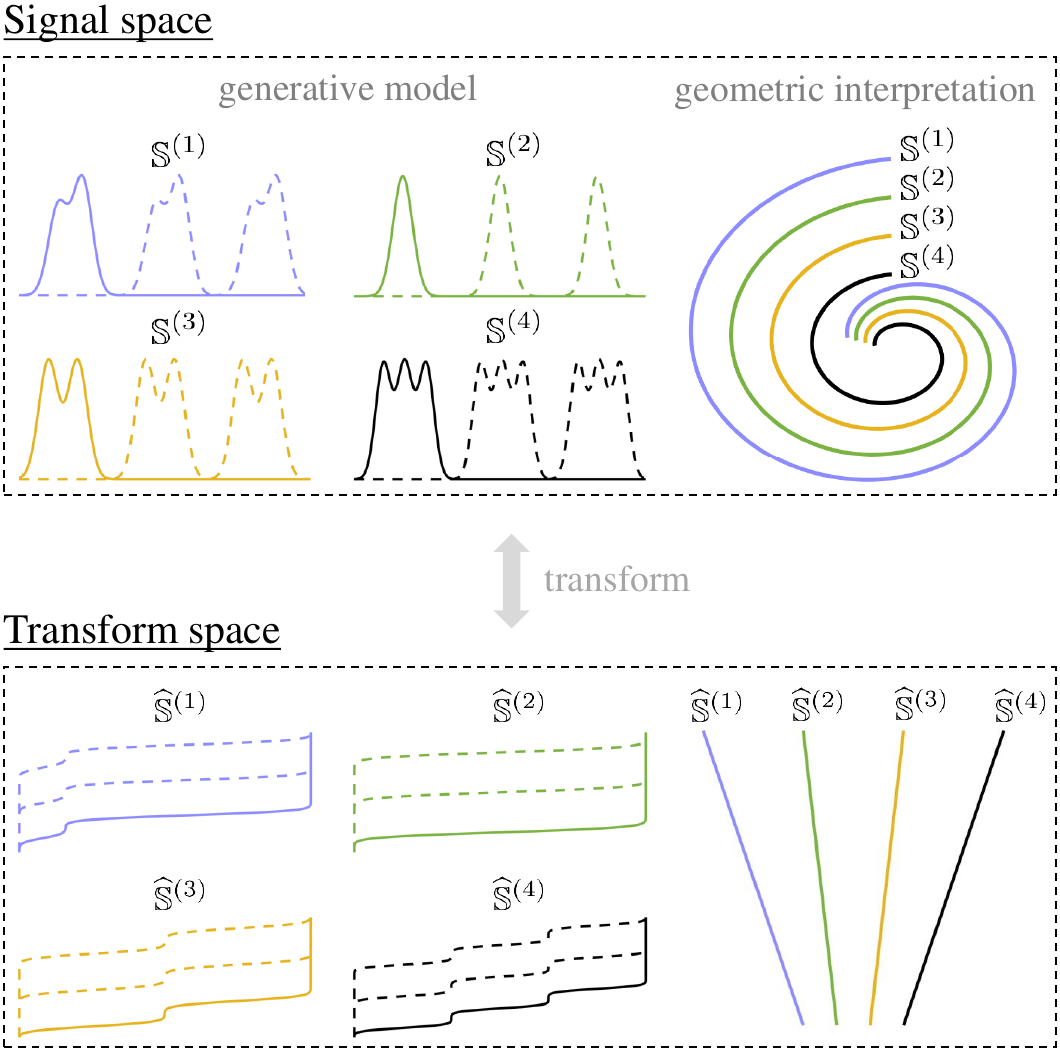}
    \caption{Generative model for signal classes in signal (top panel) and transform (bottom panel) spaces. Four classes are depicted on the left: ${\S}^{(1)}, {\S}^{(2)}, {\S}^{(3)}, {\S}^{(4)}$, each with three example signals shown. The top panel: it shows the signal classes in their corresponding native signal spaces. For each class, three example signals are shown under different translations. The right portion of the top panel shows the geometry of these four classes forming nonlinear spaces. The bottom panel: it depicts the situation in transform (CDT, or R-CDT) space. The left portion of the bottom panel shows the corresponding signals in transform domain, while the right portion shows the geometry of the signal classes forming convex spaces.}
    \label{fig:signal_transform_spaces}
    \vspace{-1em}
\end{figure}

We note that the generative model above can yield a non convex set, depending on the choice of template function $\varphi^{(k)}$ and confound category $\G$. Note that we use the same notation $\mathbb{S}^{(k)}$ for both 1D and 2D versions of the set. The meaning each time will be clear from the context.

We are now ready to define a mathematical description for a generative model-based problem statement using the definitions above:

\begin{definition}[Classification problem]
\label{def:classification}
Let $\G\subset\Diff$ and $\G$ define our set of confounds, and let $\mathbb{S}^{(k)}$ be defined as in equation~\eqref{eq:1dgenerative_model} (for signals) or equation~\eqref{eq:2dgenerative_model} (for images). Given training samples $\{s^{(1)}_1, s^{(1)}_2, \cdots\}$ (class 1), $\{s^{(2)}_1, s^{(2)}_2, \cdots\}$ (class 2), $\cdots$ as training data, determine the class $(k)$ of an unknown signal or image $s$.
\end{definition}

It is important to note that the generative model discussed yields nonconvex  (and hence nonlinear) signal classes (see Figure~\ref{fig:signal_transform_spaces}, top panel). We express this fact mathematically as: for arbitrary $s^{(k)}_i$ and $s_j^{(k)}$ we have that $\alpha s^{(k)}_i + (1-\alpha)s_j^{(k)}$, for $\alpha \in [0,1]$, may not necessarily be in $\mathbb{S}^{(k)}$. The situation is similar for images (the 2D cases). Convexity, on the other hand, means the weighted sum of samples {\it does} remain in the set; this property greatly simplifies the classification problem as will be shown in the next section.


\section{Proposed Solution}
\label{plm_soln}
We postulate that the CDT and R-CDT introduced earlier can be used to drastically simplify the solution to the classification problem posed in definition \ref{def:classification}. While the generative model discussed above generates  nonconvex (hence nonlinear) signal and image classes, the situation can change by transforming the data using the CDT (for 1D signals) or the R-CDT (for 2D images). We start by analyzing the one dimensional generative model from definition \ref{def:gen_model_1d}. 

Employing the composition property of the CDT (see Section~\ref{prelim_b}) to the 1D generative model stated in equation~\eqref{eq:1dgenerative_model} we have that
\begin{equation}
    \widehat{s}_j^{(k)} = g_j^{-1}\circ \widehat{\varphi}^{(k)} 
\end{equation} 
and thus
\begin{eqnarray}
\widehat{\mathbb{S}}^{(k)}&=&\{\widehat{s}_j^{(k)}|\widehat{s}_j^{(k)}=g_j^{-1}\circ \widehat{\varphi}^{(k)}, \forall g_j\in\GO \}. \notag
\end{eqnarray}

Thus we have the following lemma:

\begin{lemma}
\label{lemma:convexity}
If $\GO \subset \Diff $ is a convex group, the set $\widehat \S^{(k)}$ is convex. 
\end{lemma}
\begin{proof}
Let $\varphi^{(k)}$ be a template signal defined as a PDF. For $g_j\in \GO$, let $s_j^{(k)}=g_j^\prime(\varphi^{(k)}\circ g_j)$. Then using the composition property of CDT, we have that $ \widehat s_j^{(k)}=g_j^{-1}\circ \widehat \varphi^{(k)}$. Hence $\widehat{\S}^{(k)}=\{g_j^{-1}\circ \widehat \varphi^{(k)} \mid g_j\in \GO\}$. Since $\GO$ is a convex group, $\GO^{-1}$ is convex, and it follows that $\widehat{\S}^{(k)}$ is convex.

\end{proof}

\begin{figure*}[!hbt]
    \centering
    \includegraphics[width=17cm]{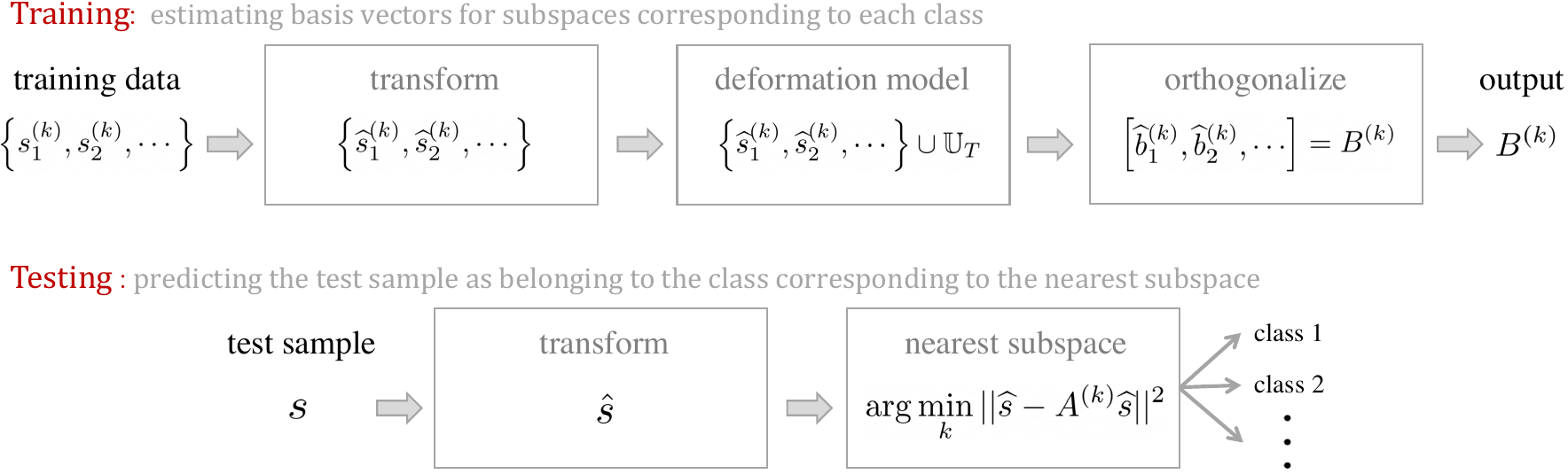}
    \caption{The training and testing process of the proposed classification model. Training: First, obtain the transform space representations of the given training samples of a particular class $(k)$. Then, enrich the space by adding the deformation spanning set $\U_T$ (see text for definition). Finally, orthogonalize to obtain the basis vectors which span the enriched space. Testing: First, obtain the transform space representation of a test sample $s$. Then, the class of $s$ is estimated to be the class corresponding to the subspace $\widehat{\S}^{(k)}_E$ which has the minimum distance $d^{2}(\widehat{s}, \widehat{\mathbb{S}}_E^{(k)})$ from $\widehat{s}$ (see text for definitions). Here, $A^{(k)}=B^{(k)}{B^{(k)}}^T$.}
    \label{train_test_fig}
    \vspace{-1em}
\end{figure*}




\begin{remark}
                                                                 Let $\mathbb{S}^{(k)}$ and $\mathbb{S}^{(p)}$ represent two generative models. If $\mathbb{S}^{(k)} \cap \mathbb{S}^{(p)}=\varnothing$, then $\widehat{\mathbb{S}}^{(k)} \cap  \widehat{\mathbb{S}}^{(p)}=\varnothing$.
    \label{lemma:nooverlap}
\end{remark}
    This follows from the fact that the CDT is a one to one map between the space of probability density functions and the space of 1D diffeomorphisms. As such the CDT operation is one to one, and therefore there exists no $\widehat{s}_j^{(k)} = \widehat{s}_i^{(p)}$.

Lemma \ref{lemma:convexity} above implies that if the set of spatial transformations formed by taking elements of $\GO$ and inverting them (denoted as $\GO^{-1}$) is convex, then the generative model will be convex in signal transform space. The situation is depicted in Figure~\ref{fig:signal_transform_spaces}. The top part shows a four class generative model that is nonlinear/non-convex. When examined in transform space, however, the data geometry simplifies in a way that signals can be added together to generate other signals in the same class -- the classes become convex in transform space.



The analysis above can be extended to the case of the 2D generative model (definition~\ref{def:gen_model_2d}) through the R-CDT. Employing the composition property of the R-CDT (see Section~\ref{prelim_d}) to the 2D generative model stated in equation~\eqref{eq:2dgenerative_model} we have that
\begin{eqnarray}
\widehat{\mathbb{S}}^{(k)}&=&\{\widehat{s}_j^{(k)}|\widehat{s}_j^{(k)}={\left(g_j^\theta\right)}^{-1}\circ \widehat{\varphi}^{(k)}, \forall g_j^\theta\in\G \}. 
\end{eqnarray}
Lemma~\ref{lemma:convexity} and Remark~\ref{lemma:nooverlap} hold true in the 2-dimensional R-CDT case as well. Thus, if $\mathcal{G}_R$ is a convex group, the R-CDT transform simplifies the data geometry in a way that image classes become convex in the R-CDT transform space. Figure~\ref{figres00}(a) depicts the situation. 

We use this information to propose a simple non-iterative training algorithm (described in more detail in Section~\ref{trn_sec}) by estimating a projection matrix that projects each (transform space) sample onto $\Vk$, for all classes $k=1,2, \cdots$, where $\Vk$ denotes the subspace generated by the convex set $\Sk$ as follows: 
\begin{align}
    \widehat{\V}^{(k)}=\mbox{span}\left(\Sk\right)=\left\{\sum_{j\in J}\alpha_j\widehat{s}_j^{(k)}\mid \alpha_j\in\R, J ~\textrm {is finite}\right\}. 
\end{align}
Figure~\ref{figres00}(b) provides a pictorial representation of $\Vk$. 
\begin{lemma}
\label{lemma:nonoverlap}
Let $\S^{(k)}, k= 1,2,...,$ be generative classes with a common confound set $\GO$ such that for any $f\notin \GO$, $f^{\prime}\varphi^{(k)}\circ f\notin \S^{(k)}$.
If $\GO$ is a convex group that also includes scaling, $\Sk\cap\widehat{\S}^{(p)}=\varnothing$, and
\begin{align}
    \label{eqlem3}\alpha ~id+(1-\alpha)h\notin\GO
\end{align} 
$\forall~ \textrm{increasing function}~h\notin\GO$ and $0<\alpha<1$ (here id denotes the identity function, $f(x)=x$), then $\Sk\cap \widehat{\V}^{(p)}=\varnothing$.
\end{lemma}

\begin{proof}
For a proof, see Appendix E in supplementary materials.
\end{proof}

Lemma~\ref{lemma:nonoverlap} above states that given certain assumptions, the convex space for a particular class does not overlap with the subspace corresponding to a different class. A corollary from Lemma \ref{lemma:nonoverlap} is that $\alpha \widehat{s}_i^{(k)}+(1-\alpha)\widehat{s}_j^{(p)}\notin \widehat{\S}^{(k)}\cup\widehat{\S}^{(p)}$ for all $\widehat{s}_i^{(k)}\in\widehat{\S}^{(k)}$ and $\widehat{s}_j^{(p)}\in\widehat{\S}^{(p)}$ with $0<\alpha<1$ (see Appendix E). Intuitively speaking, the generative classes generated by $\GO$ are "thin" in the transform space. Lemma~\ref{lemma:nonoverlap} holds true for the 2-dimensional R-CDT case as well.

There are a number examples of $\GO$ that satisfy the assumption in equation~\eqref{eqlem3}. For example, if $\GO$ is the set of translation functions, any strict convex combination (i.e., for $0<\alpha<1$) of a function other than translation and the identity function, is not a translation function either. One can also verify that there are other sets of functions that also satisfy the assumption, e.g., the set of increasing affine functions, the set of diffeomorphisms that have a common fixed point, etc.

It follows from Lemma~\ref{lemma:nonoverlap} that, if the test sample was generated according to the generative  model for one of the classes, then there will exist exactly one class $(k)$ for which $d^2( \widehat{s},\widehat{\mathbb{S}}^{(k)})= d^2( \widehat{s},\Vk) = 0$. It also follows, $d^2( \widehat{s},\widehat{\V}^{(p)})>0$ when $k\neq p$ \footnote{Rigorously speaking, if $\widehat{\V}^{(p)}$ is a closed subspace, then $d^2( \widehat{s},\widehat{\V}^{(p)})>0$ if and only if $\widehat s\notin \widehat{\V}^{(p)} $. In practice, $\widehat{\V}^{(p)}$ will be a finite dimensional space and hence the closedness condition is satisfied.}. Here $d^2(\cdot,\cdot)$ is the Euclidean distance between $\widehat{s}$ and the nearest point in $\widehat{\mathbb{S}}^{(k)}$ or $\Vk$. 

As far as a test procedure for determining the class of some unknown signal or image $s$, under the assumption that $\Sk \cap  \widehat{\V}^{(p)}=\varnothing$, it then suffices to measure the distance between $\widehat{s}$ and the nearest point in each subspace $\Vk$ corresponding to the generative model $\widehat{\S}^{(k)}$. Therefore, under the assumption that the testing sample at hand $s$ was generated according to one of the (unknown) classes as described in definition \ref{def:classification}, the class of the unknown sample can be decoded by solving
\begin{equation}
     \arg \min_k d^2(\widehat{s},\Vk).
    \label{eq:nearest_subspace}
\end{equation}

Finally, note that due to property~II-D.2 we also have that
\begin{equation}
    d^2(\widehat{s},\widehat{\mathbb{S}}^{(k)}) = \min_{g^\theta} SW^2_2\left(s,\mathscr{R}^{-1}\left(\left({g^\theta}\right)^\prime \widetilde{\varphi}^{(k)} \circ g^\theta\right)\right)\nonumber
\end{equation}
with $g^\theta \in \G$. In words, the R-CDT nearest subspace method proposed in equation~\eqref{eq:nearest_subspace} can be considered to be equivalent to a nearest (in the sense of the sliced-Wasserstein distance) subset method in image space, with the subset given by the generative model stated in definition \ref{def:gen_model_2d}.

\subsection{Training algorithm}
\label{trn_sec}
Using the principles and assumptions laid out above, the algorithm we propose estimates the subspace $\Vk$ corresponding to the transform space $\Sk$ given sample data $\{s_1^{(k)}, s_2^{(k)}, \cdots \}$. Naturally, the first step is to transform the training data to obtain $\{\widehat{s}_1^{(k)}, \widehat{s}_2^{(k)}, \cdots \}$. We then approximate $\Vk$ as follows: 
\begin{align}
    \Vk=\text{span}\left\{ \widehat{s}_1^{(k)}, \widehat{s}_2^{(k)}, \cdots  \right \}.\nonumber
\end{align}


Given the composition properties for the CDT and R-CDT (see properties II-B.1 and II-D.1), it is also possible to enrich $\Vk$ in such a way that it will automatically include the samples undergoing some specific deformations without explicitly training with those samples under said deformation. The spanning sets corresponding to two such deformations, image domain translation and isotropic scaling, are derived below: 
\begin{itemize}
	\item[i)] Translation: let $g(\mathbf{x})= \mathbf{x}-\mathbf{x_0}$ be the translation by $\mathbf{x}_0\in \R^2$ and $s_g(\mathbf{x})=|\det Jg|s\circ g=s(\mathbf{x}-\mathbf{x_0})$. Note that $Jg$ denotes the Jacobian matrix of $g$. Following \cite{kolouri2016radon} we have that $\widehat s_g (t,\theta) =\widehat{s}(t,\theta)+\mathbf{x}_0^T \xi_\theta$ where $\xi_\theta = [\cos(\theta),\sin(\theta)]^T$. We define the spanning set for translation in transform domain as $\U_T=\{u_1(t,\theta),u_2(t,\theta)\}$, where $u_1(t,\theta)=\cos\theta$ and $u_2(t,\theta)=\sin\theta$. 
	\item[ii)]Isotropic scaling: let $g(\mathbf{x})=\alpha\mathbf{x}$  and $s_g(\mathbf{x})=|Jg|s\circ g=\alpha^2s(\alpha\mathbf{x})$, which is the normalized dilatation of $s$ by $\alpha$ where $\alpha\in\R_{+}$. Then according to \cite{kolouri2016radon}, $\widehat s_g (t,\theta) = \widehat{s} (t,\theta)/\alpha$, i.e. a scalar multiplication. Therefore, an additional spanning set is not required here and thereby the spanning set for isotropic scaling becomes $\U_D=\varnothing$.
\end{itemize}
Note that the spanning sets are not limited to translation and isotropic scaling only. Other spanning sets might be defined as before for other deformations as well. However, deformation spanning sets other than translation and isotropic scaling are not used here and left for future exploration. 

In light of the above discussion, we define the enriched space $\Vk_E$ as follows:
\begin{align}
    \label{eq:transscal} \Vk_E=\mbox{span}\left(\left\{\widehat{s}_1^{(k)},\widehat{s}_2^{(k)},\cdots\right\}\cup\U_T\right)
\end{align}
where $\U_T=\{u_1(t,\theta),u_2(t,\theta)\}$, with $u_1(t,\theta)=\cos\theta$ and $u_2(t,\theta)=\sin\theta$. Figure~\ref{figres00}(b) depicts this situation.

We remark that although the R-CDT transform \eqref{eq:rcdt} is introduced in a continuous setting, numerical approximations for both the Radon and CDT transforms are available for discrete data, i.e., images in our applications \cite{kolouri2016radon}. Here we utilize the computational algorithm described in \cite{park2018cumulative} to estimate the CDT from observed, discrete data. Using this algorithm, and given an image $s$, $\widehat s$ is computed on a chosen grid  $[t_1,...,t_m]\times [\theta_1,...,\theta_n]$ and reshaped as a vector in $\R^{mn}$.\footnote{The same grid is chosen for all images. $m,n$ are positive integers .} Also the elements in $\U_T$ were computed on the above grid and reshaped to obtain a set of vectors in $\R^{mn}$. 

Finally, the proposed training algorithm includes the following steps: for each class $k$

\begin{enumerate}

\item Transform training samples to obtain $\left\{ \widehat{s}_1^{(k)}, \widehat{s}_2^{(k)}, \cdots  \right \}$ 

\item Orthogonalize $\left\{ \widehat{s}_1^{(k)}, \widehat{s}_2^{(k)}, \cdots  \right \}\cup\U_T$ to obtain the set of basis vectors $\left\{b_1^{(k)},b_2^{(k)},\cdots\right\}$, which spans the space $\Vk_E$ (see equation~\eqref{eq:transscal}). Use the output of orthogonalization procedure to define the matrix $B^{(k)}$ that contains the basis vectors in its columns as follows:
\begin{equation}
    B^{(k)} = \begin{bmatrix}b^{(k)}_1, b^{(k)}_2, \cdots \end{bmatrix}\notag
\end{equation}
\end{enumerate}
The training algorithm described above is summarized in Figure~\ref{train_test_fig}.

\subsection{Testing algorithm}
The testing procedure consists of applying the R-CDT transform followed by a nearest subspace search in  R-CDT space (see Figure~\ref{figres00}(c)). Let us consider a testing image $s$ whose class is to be predicted by the classification model described above. As a first step, we apply R-CDT on $s$ to obtain the transform space representation $\widehat{s}$. We then estimate the distance between $\widehat{s}$ and the subspace model for each class by $d^{2}(\widehat{s}, \Vk_E)\sim \| \widehat{s} - B^{(k)}{B^{(k)}}^T\widehat{s} \|^2$. Note that $B^{(k)}{B^{(k)}}^T$ is an orthogonal projection matrix onto the space generated by the span of the columns of $B^{(k)}$ (which form an orthogonal basis). To obtain this distance, we must first obtain the projection of $\widehat{s}$ onto the nearest point in the subspace $\Vk_E$, which can be easily computed by utilizing the orthogonal basis $\left\{b^{(k)}_1, b^{(k)}_2, \cdots \right\}$ obtained in the training algorithm. Although the pseudo-inverse formula could be used, it is advantageous in testing to utilize an orthogonal basis for the subspace instead. The class of $\widehat{s}$ is then estimated to be
\begin{equation}
    \arg \min_k \| \widehat{s} - A^{(k)} \widehat{s} \|^2.\nonumber
\end{equation}
where, $A^{(k)}=B^{(k)}{B^{(k)}}^T$. Figure~\ref{train_test_fig} shows a system diagram outlining these steps.

\section{Computational Experiments}
\label{mnm}

\subsection{Experimental setup}
Our goal is to study the classification performance of the method outlined above with respect to state of the art techniques (deep CNN's), and in terms of metrics such as classification accuracy, computational complexity, and amount of training data needed. Specifically, for each dataset we study, we generated train-test splits of different sizes from the original training set, trained the models on these splits, and reported the performances on the original test set. For a train split of a particular size, its samples were randomly drawn (without replacement) from the original training set, and the experiments for this particular size were repeated 10 times. All algorithms saw the same train-test data samples for each split. Apart from predictive performances, we also measured different models' computational complexity, in terms of total number of floating point operations (FLOPs).

A particularly compelling property of the proposed approach is that the R-CDT subspace model can capture different sizes of deformations (e.g. small translations vs. large translations) without requiring that all such small and large deformations be present in the training set. In other words, our model generalizes to data distributions that were previously unobserved. This is a highly desirable property particularly for applications such as the optical communication under turbulence problem described below, where training data encompassing the full range of possible deformations are limited. This property will be explored in section \ref{result}. 


Given their excellent performance in many classification tasks, we utilized different kinds of neural network methods as a baseline for assessing the relative performance of the method outlined above. Specifically, we tested three neural network models: 1) a shallow CNN model consisting of two convolutional layers and two fully connected layers (based on PyTorch's official MNIST demonstration example), 2) the standard VGG11 model~\cite{simonyan2014very}, and 3) the standard Resnet18 model~\cite{he2016deep}. All these models were trained for 50 epochs, using the Adam~\cite{kingma2014adam} optimizer with learning rate of 0.0005. When the training set size was less than or equal to 8, a validation set was not used, and the test performance was measured using the model after the last epoch. When the training set had more than 8 samples we used 10\% of the training samples for validation, and reported the test performance based on the model that had the best validation performance. To make a fair comparison we did not use data augmentation in the training phase of the neural network models nor of the proposed method.

The proposed method was trained and tested using the methods explained in section~\ref{plm_soln}. The orthogonalization of $\Vk_E$ was performed using singular value decomposition (SVD). The matrix of basis vectors $B^{(k)}$ was constructed using the left singular vectors obtained by the SVD of $\Vk_E$. The number of the basis vectors was chosen in such a way that the sum of variances explained by all the selected basis vectors in the $k$-th class captures 99\% of the total variance explained by all the training samples in the $k$-th class. A 2D uniform probability density function was used as the reference image for R-CDT computation (see equation~\eqref{eq:rcdt}).

\vspace{-0.5em}

\subsection{Datasets}
\begin{table}[]
\caption{Datasets used in the experiment.}
\centering
\label{dataset_table1}
\begin{tabular}{lcccc}
\hline
              & Image size & \makecell{No. of \\classes} & \makecell{No. of \\training \\images} & \makecell{No. of \\test \\images} \\ \hline
Chinese printed character     & 64 $\times$ 64          & 1000             & 100000                             & 100000                         \\
MNIST         & 28 $\times$ 28          & 10             & 60000                             & 10000                         \\
Affine-MNIST  & 84 $\times$ 84          & 10             & 60000                             & 10000                         \\
Optical OAM   & 151 $\times$ 151          & 32             & 22400                             & 9600                         \\
Sign language & 128 $\times$ 128          & 3             & 3280                             & 1073                         \\
OASIS brain MRI       & 208 $\times$ 208          & 2             & 100                             & 100                         \\
CIFAR10       & 32 $\times$ 32          & 10             & 50000                             & 10000                         \\
\hline
\end{tabular}
\vspace{-1em}
\end{table}
To demonstrate the comparative performance of the proposed method, we identified seven datasets for image classification: Chinese printed characters, MNIST, Affine-MNIST, optical OAM, sign language, OASIS Brain MRI, and CIFAR10 image datasets. The Chinese printed character dataset with 1000 classes was created by adding random translations and scalings to the images of 1000 printed Chinese characters. The MNIST dataset contains images of ten classes of handwritten digits which was collected from \cite{lecun1998gradient}. The Affine-MNIST dataset was created by adding random translations and scalings to the images of the MNIST dataset. The optical orbital angular momentum (OAM) communication dataset was collected from \cite{park2018multiplexing}. The dataset contains images of 32 classes of multiplexed oribital angular momentum beam patterns for optical communication which were corrupted by atmospheric turbulence. The sign language dataset was collected from \cite{sign_lang} which contains images of hand gestures. Normalized HOGgles images \cite{vondrick2013hoggles} of first three classes of the original RGB hand gesture images were used. Finally, the OASIS brain MRI image dataset was collected from \cite{marcus2007open}. The 2D images from the middle slices of the the original 3D MRI data were used in this paper. Besides these six datasets, we also demonstrated the results on the natural images of the gray-scale CIFAR10 dataset \cite{krizhevsky2009learning}. The details of the seven datasets used are available in Table~\ref{dataset_table1}.



\section{Results}
\label{result}

\subsection{Test accuracy}
\begin{figure*}[htb]
    \centering
    \includegraphics[width=18cm]{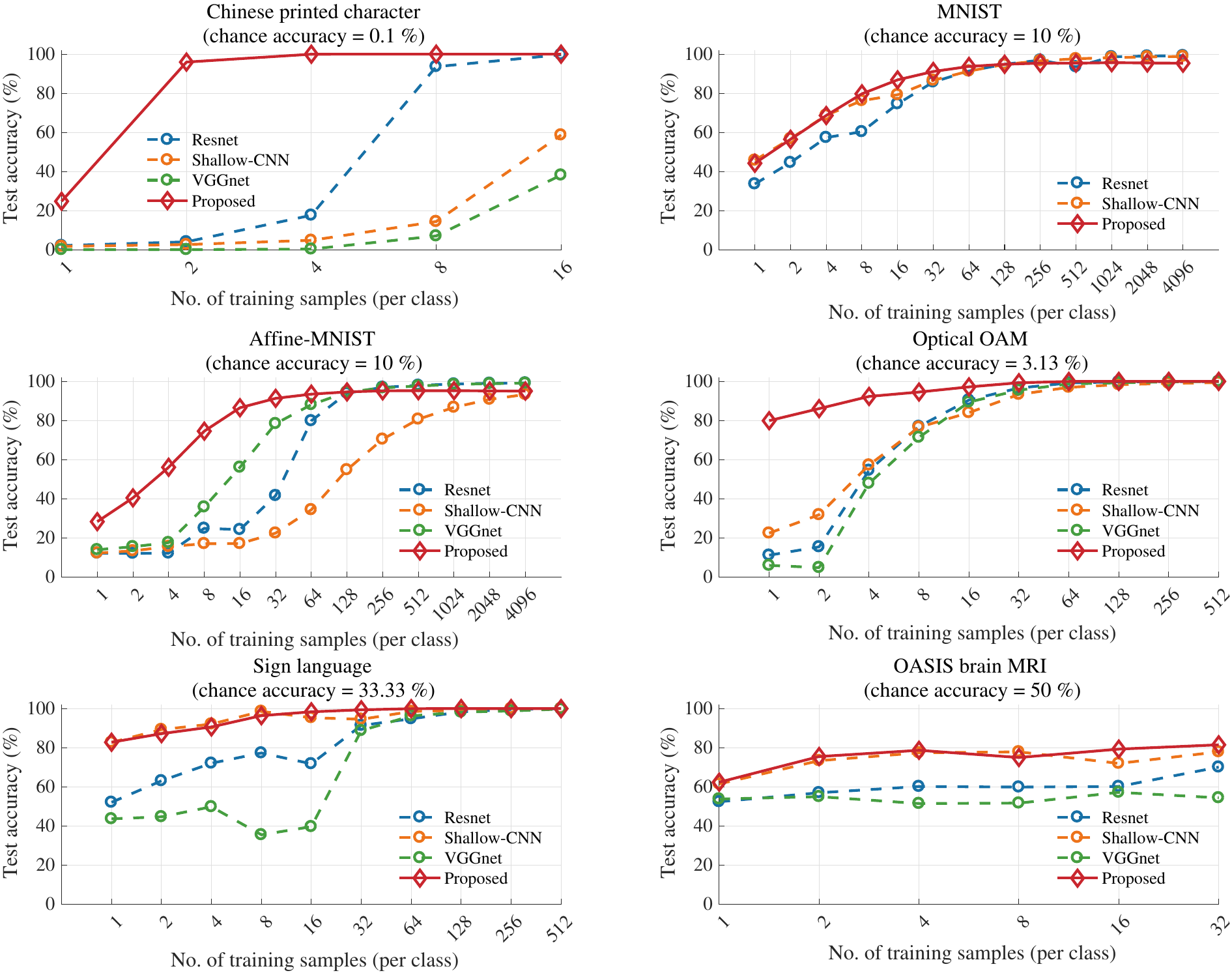}
    \caption{Percentage test accuracy of different methods as a function of the number of training images per class.}
    \label{figres01}
    \vspace{-1em}
\end{figure*}
\begin{figure*}[!htb]
    \centering
    \includegraphics[width=18cm]{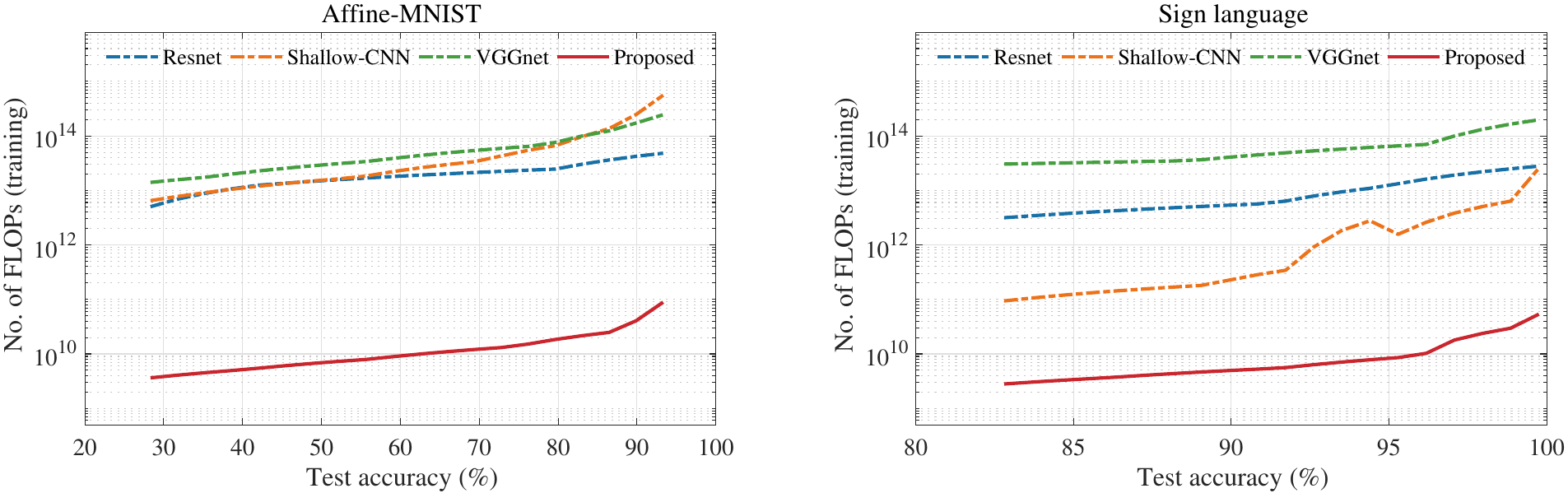}
    \caption{The total number of floating point operations (FLOPs) required by the methods to attain a particular test accuracy in the MNIST dataset (left) and the sign language dataset (right).}
    \label{figres05}
    \vspace{-1.0em}
\end{figure*}
\begin{figure*}[!htb]
    \centering
    \includegraphics[width=18cm]{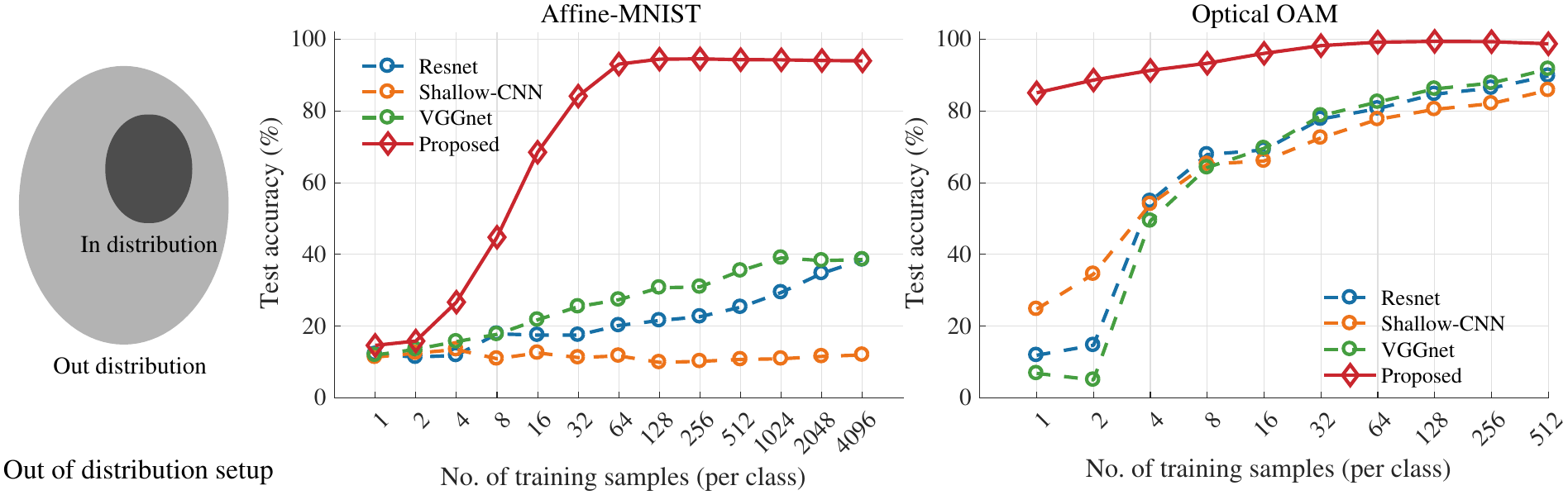}
    \caption{Computational experiments under the out-of-distribution setup. The out-of-distribution setup consists of disjoint training (`in distribution') and test (`out distribution') sets containing different sets of magnitudes of the confounding factors (see the left panel). Percentage test accuracy of different methods are measured as a function of the number of training images per class under the out-of-distribution setup (see the middle and the right panel).}
    \label{figres02}
    \vspace{-1em}
\end{figure*}
The average test accuracy values of the methods tested on Chinese printed character, MNIST, Affine-MNIST, optical OAM, sign language, and OASIS brain MRI image datasets for different number of training samples per class are shown in Figure~\ref{figres01}. Note that we did not use VGG11 in the MNIST dataset because the dimensions of MNIST images ($28\times28$, see Table~\ref{dataset_table1}) are too small for VGG11.

Overall, the proposed method outperforms other methods when the number of training images per class is low (see Figure~\ref{figres01}). For some datasets, the improvements are strikingly significant. For example, in the optical OAM dataset, and for learning from only one sample per class, our method provides an absolute improvement in test accuracy of $\sim 60\%$ over the CNN-based techniques. Also, the proposed method offers comparable performance to its deep learning counterparts when increasing the number of training samples. 

Furthermore, in most cases, the accuracy vs. training size curves have a smoother trend in the proposed method as compared with that of CNN-based learning. The standard deviation of test accuracy of the proposed method is also lower than the other methods in most of the cases (see Appendix F in supplementray materials). Moreover, the accuracy vs. training curves of the neural network architectures significantly vary as a function of the choice of the dataset. For example, Shallow-CNN outperforms Resnet in MNIST dataset while it underperforms Resnet in Affine-MNIST dataset in terms of test accuracy. Again, while outperforming VGG11 in the sign language dataset, the Resnet architecture underperforms VGG11 in the Affine-MNIST dataset. 

\subsection{Computational efficiency}
Figure~\ref{figres05} presents the number of floating point operations (FLOPs) required in the training phase of the classification models in order to achieve a particular test accuracy value. We used the Affine-MNIST and the sign language datasets in this experiment. 

The proposed method obtains test accuracy results similar to that of the CNN-based methods with $\sim50$ to $\sim10,000$ times savings in computational complexity, as measured by the number of FLOPs (see Figure~\ref{figres05}). The reduction of the computational complexity is generally larger when compared with a deep neural network, e.g., VGG11. The number of FLOPs required by VGG11 is $\sim3,000$ to $\sim10,000$ times higher than that required by the proposed method, whereas Shallow-CNN is $\sim50$ to $\sim6,000$ times more computationally expensive than the proposed method in terms of number of FLOPs. Note that, we have included the training FLOPs only in Figure~\ref{figres05}. We also calculated the number of FLOPs required in the testing phase. For all the methods, the number of test FLOPs per image is approximately $5$ orders of magnitude ($\sim10^5$) lower than the number of training FLOPs. The testing FLOPs of the proposed method depend on the number of training samples. Despite this fact, the number of test FLOPs required by the CNN-based methods in our experiments is $\sim5$ to $\sim100$ times more than the maximum number of test FLOPs required by the proposed method. These plots are not shown for brevity.
\vspace{-0.75em}
\subsection{Out-of-distribution testing}
In this experiment, we varied the magnitude of the confounding factors (e.g., translation) to generate a gap between training and testing distributions that allows us to test the out-of-distribution performance of the methods. Formally, let $\G \subset\Diff$ define the set of confounding factors. Let us consider two disjoint subsets of $\G$, denoted as $\GO_{in}$ and $\GO_{out}$, such that $\GO_{in}\subset\G$ and $\GO_{out}=\G\backslash\GO_{in}$. Using the generative model in equation~\eqref{eq:2dgenerative_model} the `in distribution' image subset $\mathbb{S}_{in}^{(k)}$ and the `out distribution' image subset $\mathbb{S}_{out}^{(k)}$ are defined using the two disjoint confound subsets $\GO_{in}$ and $\GO_{out}$ as follows:
\begin{align}
&\mathbb{S}_{in}^{(k)}=\left\{s_j^{(k)}|s_j^{(k)}=\RR^{-1}\left(\left({g_j^\theta}\right)^\prime\widetilde{\varphi}^{(k)}\circ g^\theta_j\right), \forall g^\theta_j\in\GO_{in} \right\}\nonumber\\
&\mathbb{S}_{out}^{(k)}=\left\{s_j^{(k)}|s_j^{(k)}=\RR^{-1}\left(\left({g_j^\theta}\right)^\prime\widetilde{\varphi}^{(k)}\circ g^\theta_j\right), \forall g^\theta_j\in\GO_{out} \right\}\nonumber
\end{align}
We defined the `in distribution' image subset $\mathbb{S}_{in}^{(k)}$ as the generative model for the training set and the `out distribution' image subset $\mathbb{S}_{out}^{(k)}$ as the generative model for the test set in this modified experimental setup (see the left panel of Figure~\ref{figres02}).

We measured the accuracy of the methods on the Affine-MNIST and the optical OAM datasets under the modified experimental setup. The Affine-MNIST dataset for the modified setup was generated by applying random translations and scalings to the original MNIST images in a controlled way so that the confound subsets $\GO_{in}$ and $\GO_{out}$ do not overlap. The `in distribution' image subset $\mathbb{S}^{(k)}_{in}$ consisted of images with translations by not more than $7$ pixels and scale factors varying between $0.9\sim 1.2$. On the other hand, images with translations by more than $7$ pixels and scale factors varying between $1.5\sim 2.0$ were used to generate the `out distribution' image subset $\mathbb{S}^{(k)}_{out}$. For the optical OAM dataset, the images at turbulence level 5 (low turbulence) \cite{park2018multiplexing} were included in the `in distribution' subset $\mathbb{S}^{(k)}_{in}$ and those at turbulence level 10 and 15 (medium and high turbulence) were included in the `out distribution' subset $\mathbb{S}^{(k)}_{out}$. The average test accuracy results for different training set sizes under the out-of-distribution setup are shown in Figure~\ref{figres02}. 

The proposed method outperforms the other methods by a greater margin than before under this modified experimental scenario (see Figure~\ref{figres02}). For the Affine-MNIST dataset, the test accuracy values of the proposed method are $\sim2$ to $\sim85\%$ higher than that of the CNN-based methods. For the optical OAM dataset, the accuracy values of the proposed method are $\sim7$ to $\sim85\%$ higher than those of the CNN-based methods (see Figure~\ref{figres02}). 

\begin{figure*}[htb]
    \centering
    \includegraphics[width=18cm]{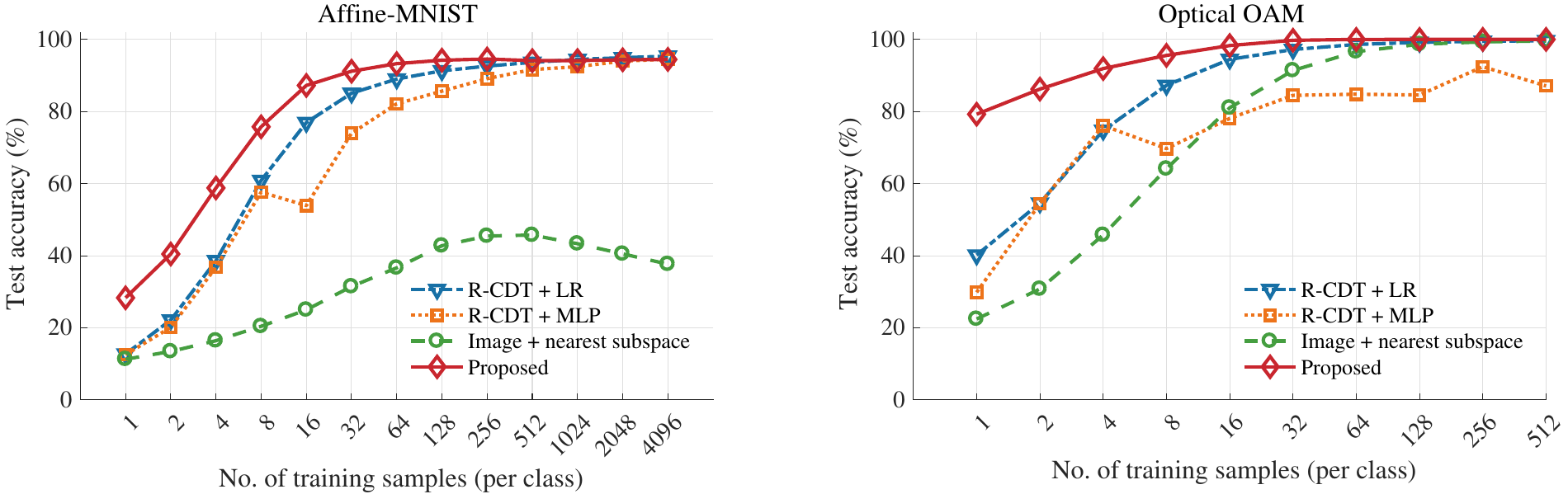}
    \caption{Comparison of the percentage test accuracy results obtained in the three ablation studies conducted (using the MLP-based and LR classifiers in R-CDT space and the nearest subspace classifier in image space) with that of the proposed method.}
    \label{figres06}
    \vspace{-0.5em}
\end{figure*}

As compared with the general experimental setup (Figure~\ref{figres01}), the test accuracy results of all the methods mostly reduce under this challenging modified experimental setup (Figure~\ref{figres02}). The average reduction of test accuracy of the proposed method under the modified setup is also significantly lower than that of the CNN-based methods. For the Affine-MNIST dataset, the average reduction of test accuracy for the proposed method is $\sim10\%$. Whereas, the reduction of test accuracy for the CNN-based methods are $\sim36\%-42\%$. Similarly, for the optical OAM dataset, the average reduction of accuracy are $\sim0\%$ and $\sim 9\%-12\%$ for the proposed method and the CNN-based methods, respectively.


\vspace{-0.5em}
\subsection{Ablation study}
To observe the relative impact of different components of our proposed method, we conducted three ablation studies using the Affine-MNIST and the optical OAM datasets. In the first two studies, we replaced the nearest subspace-based classifier used in our proposed method with a multilayer perceptron (MLP) \cite{gardner1998artificial} and a logistic regression (LR) classifier \cite{lr2020}, respectively, and measured the test accuracy of these modified models. In the third study, we replaced the R-CDT transform representations with the raw images. We measured the test accuracy of the nearest subspace classifier used with the raw image data. The percentage test accuracy results obtained in these modified experiments are illustrated in Figure~\ref{figres06} along with the results of the proposed method for comparison. The proposed method outperforms all these modified models in terms of test accuracy (see Figure~\ref{figres06}).
\begin{figure}[hbt]
    \centering
    \includegraphics[width=8.5cm]{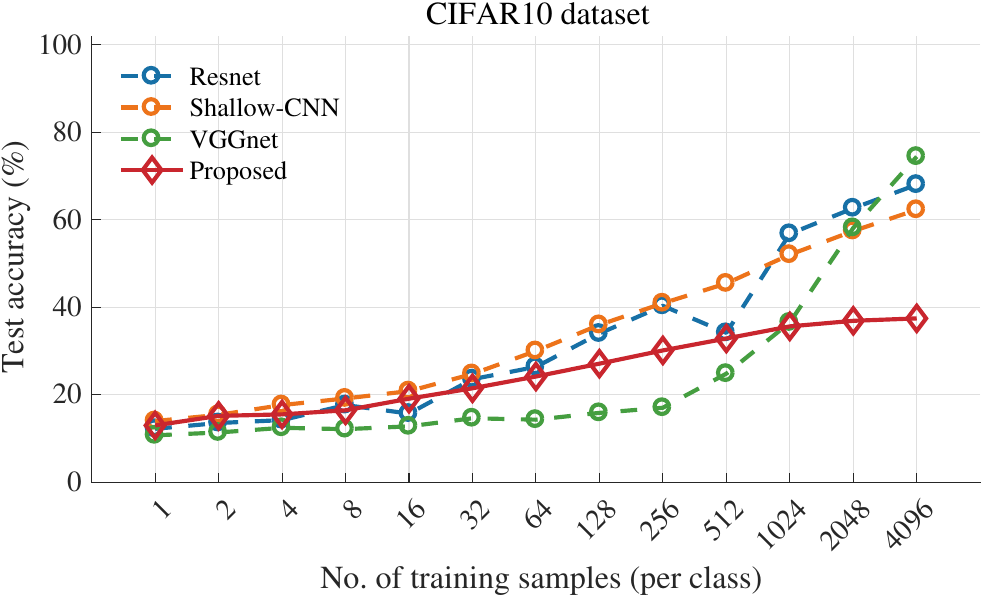}
    \caption{Percentage test accuracy results in the CIFAR10 dataset. The natural images in the CIFAR10 dataset might not conform to the underlying generative model, and therefore, the proposed method doesnot perform well in the CIFAR10 dataset.}
    \label{figres04}
    \vspace{-1em}
\end{figure}
\vspace{-1em}
\subsection{An example where the proposed method fails}
There are examples of image classification problems (e.g. natural image dataset) where the proposed method does not perform well. One such example of this kind of dataset is CIFAR10 dataset. To demonstrate this point, we measured the test accuracies of different methods on the gray-scale CIFAR10 dataset (see Figure~\ref{figres04}). It can be seen that, the highest accuracy of the proposed method is lower than the CNN-based methods. All of the CNN-based methods used outperform the proposed method in the gray-scale CIFAR10 dataset in terms of maximum test accuracy.

\section{Discussion}
\label{discuss}

\subsection*{Test accuracy}
Results shown with 6 example datasets suggest the proposed method obtains competitive accuracy figures as compared with state of the art techniques such as CNNs as long as the data at hand conform to the generative model in equation~\eqref{eq:2dgenerative_model}. Moreover, in these examples, the nearest R-CDT subspace method was shown to be more data efficient: generally speaking, it can achieve higher accuracy with fewer training samples. 


\subsection*{Computational efficiency}
The proposed method obtains accuracy figures similar to that of the CNN-based methods with $\sim50$ to $\sim10,000$ times reduction of the computational complexity. Such a drastic reduction of computation can be achieved due to the simplicity and non-iterative nature of the proposed solution. As opposed to the neural networks where GPU implementations are imperative, the proposed method can efficiently be implemented in a CPU and greatly simplify the process of obtaining an accurate classification model for the set of problems that are well modeled by our problem statement defined in definition~\ref{def:classification}. 




\subsection*{Out-of-distribution testing}
The accuracy results of the CNN-based methods drastically fall under the out-of-distribution setting whereas the proposed method maintains its test accuracy performance. Based on the above findings we infer that the proposed method can be suitable for both interpolation (predicting the classes of data samples within the known distribution) and extrapolation (predicting the classes of data samples outside the known distribution) when the data conforms to the generative model expressed in definition \ref{def:gen_model_2d}.

The out-of-distribution setting for image classification also bears practical significance. For example, consider the problem of classifying the OAM beam patterns for optical communications (see the optical OAM dataset in Figure~\ref{figres01}). As these optical patterns traverse air with often unknown air flow patterns, temperature, humidity, etc., exact knowledge of the turbulence level that generated a test image may not always be at hand. Therefore, it is practically infeasible to train the classification model with images at the same turbulence level as the test data. The out-of-distribution setup is more practical under such circumstances. 
\vspace{-1.0em}
\subsection*{Ablation study}
Based on the ablation study results, we conclude that the proposed method of using the nearest subspace classifier in R-CDT domain is more appropriate for the category of classification problems we are considering. Data classes in original image domain do not generally form a convex set and therefore and in these instances the subspace model is not appropriate in image domain. The subspace model is appropriate in R-CDT domain as the R-CDT transform provides a linear data geometry. Considering the subspace model in R-CDT space also enhances the generative nature of the proposed classification method by implicitly including the data points from the convex combination of the given training data points. Use of a discriminative model for classification (e.g., MLP, LR, etc.) with the R-CDT domain representations of images does not have that advantage.
\vspace{-1em}

\subsection*{When are $\GO^{-1}$ and $\G^{-1}$ convex}
Given the performance in terms of accuracy and complexity, the R-CDT subspace model presented above seems to be an appropriate model for many applications. However, that is not always the case, as the results with the CIFAR10 dataset show. It is thus natural to ask for what types of problems will the proposed method work well.

The definitions expressed in \ref{def:gen_model_1d} and \ref{def:gen_model_2d} define the generative model for the data classes used in our classification problem statement \ref{def:classification}. As part of the solution to the classification problem, it was proved in Lemma \ref{lemma:convexity} that so long as $\GO^{-1}$ or $\G^{-1}$ (the inverse of the transportation subset of functions) is convex, $\Sk$ is convex, and that is a precondition for the proposed classification algorithm summarized in Figure~\ref{train_test_fig} to solve the classification problem stated in \ref{def:classification}. A natural question to ask is when, or for what types of transportation functions is this condition met? Certain simple examples are easy to describe. For example, when $\GO$ or $\G$ denotes the set of translations in 1 or 2D, then $\GO^{-1}$ or $\G^{-1}$ can be shown to be convex. Furthermore, when $\GO$ or $\G$ refers to the set of scalings of a function, then $\GO^{-1}$ or $\G^{-1}$ can be shown to be convex. When $\GO$ or $\G$ contains a set of fixed points, i.e. when $g(t_i) = t_i$, then $\GO^{-1}$ or $\G^{-1}$ can be shown to be convex. Our hypothesis is that the 6 problems we tested the method on conform to the generative model specifications at least in part, given that classification accuracies significantly higher than chance are obtained with the method. A careful mathematical analysis of these and related questions is the subject of present and future work.

\subsection*{Limitation: An example where the proposed method fails}
The fundamental assumption of the proposed method is that the data at hand conform to an underlying generative model (equation~\ref{eq:2dgenerative_model}). If the dataset does not conform to the generative model, the proposed method may not perform well. The CIFAR10 dataset (Figure~\ref{figres04}) is an example where the data classes might not follow the generative model. The proposed method underperforms the CNN-based methods in the case of the CIFAR10 dataset.

\section{Conclusions}
\label{conclude}
We introduced a new algorithm for supervised image classification. The algorithm builds on prior work related to the Radon Cumulative Distribution Transform (R-CDT) \cite{kolouri2016radon} and classifies a given image by measuring the the distance between the R-CDT of that image and the linear subspaces $\Vk$, $k=1,2,\cdots, N_{\mbox{classes}}$ estimated from the linear combination of the transformed input training data. As distances between two images in R-CDT space equate to the sliced Wasserstein distances between the inverse R-CDT of the same points, the classification method can be interpreted as a `nearest' Sliced Wasserstein distance method between the input image and other images in the generative model $\S^{(k)}$ for each class $k$. 


The model was demonstrated to solve a variety of real-world classification problems with accuracy figures similar to state of the art neural networks including a shallow method, VGG11~\cite{simonyan2014very}, and a Resnet18~\cite{he2016deep}. The proposed model was also shown to outperform the neural networks by a large margin in some specific practical scenarios, e.g., training with very few training samples and testing with `out of distribution' test sets. The method is also extremely simple to implement, non-iterative, and it does not require tuning of hyperparameters. Finally, as far as training is concerned the method was also demonstrated to be significantly less demanding in terms of floating point operations relative to different neural network methods.

We note, however, that the method above is best suited for problems that are well modeled by the generative model definition provided in Section \ref{plm_state}. The definition is naturally tailored towards modeling images which are segmented (foreground extracted). Examples shown here include classifying written Chinese characters, MNIST numerical digits, optical communication patterns, sign language hand shapes, and brain MRIs.  We also note that the model does not account for many other variations present in many important image classification problems. Specifically, the proposed model does not account for occlusions, introduction of other objects in the scene, or variations which cannot be modeled as a mass (intensity) preserving transformation on a set of templates. Computational examples using the CIFAR10 dataset demonstrate that indeed the proposed model lags far behind, in terms of classification accuracy, the standard deep learning classification methods to which it was compared.

Finally, we note that numerous adaptations of the method are possible. We note that the linear subspace method (in R-CDT space) described above can be modified to utilize other assumptions regarding the set that best models each class. While certain classes or problems may benefit from a simple linear subspace method as described above, where all linear combinations are allowed, other classes may be composed by the union of non-orthogonal subspaces. Furthermore, note that, we focus on supervised learning in this paper. The method can however be adapted to be used in the context of unsupervised learning also (subspace clustering, for example). The exploration of this and other modifications and extensions of the method are left for future work.
\vspace{-1.0em}
\section*{Acknowledgments}
This work was supported in part by NIH grants GM130825, GM090033.
\vspace{-1.0em}


\clearpage
\setcounter{page}{1} 
\appendices
\setcounter{equation}{0}
\section{Proof of property II-B.1.}
\subsection*{The composition property of the CDT:} Let $s(x)$ denote a normalized signal and let $\widehat{s}(x)$ be the CDT of $s(x)$. The CDT of $s_g=g's\circ g$ is given by 
\begin{align}
    \widehat{s}_g=g^{-1}\circ\widehat{s}\nonumber
\end{align}

\begin{proof}
Let $r$ denote a reference signal. If $\widehat{s}$ and $\widehat{s}_g$ denote the CDTs of $s$ and $s_g$, respectively, with respect to the reference $r$, we have that
\begin{align}
    \int_{-\infty}^{\widehat{s}(x)}s(u)du=\int_{-\infty}^{\widehat{s}_g(x)}s_g(u)du=\int_{-\infty}^{x}r(u)du\nonumber
\end{align}
By substituting $s_g=g's\circ g$ we have
\begin{align}
    \label{appeq2}\int_{-\infty}^{\widehat{s}(x)}s(u)du=\int_{-\infty}^{\widehat{s}_g(x)}g^\prime(u)s(g(u))du\numberwithin{equation}{section}
\end{align}
By the change of variables theorem, we can replace $g(u)=v$, $g^\prime(u)du=dv$ in equation~\eqref{appeq2}:
\begin{align}
    \label{appeq3}\int_{-\infty}^{\widehat{s}(x)}s(u)du=\int_{-\infty}^{g\left(\widehat{s}_g(x)\right)}s(v)dv\numberwithin{equation}{section}
\end{align}
From equation~\eqref{appeq3}, we have that
\begin{align}
    g\left(\widehat{s}_g(x)\right)=\widehat{s}(x)\implies\widehat{s}_g(x)=g^{-1}\left(\widehat{s}(x)\right)~\mbox{or,}~\widehat{s}_g=g^{-1}\circ\widehat{s} \nonumber
\end{align}
\end{proof}

\section{Proof of property II-D.1.}
\subsection*{The composition property of the R-CDT:} Let $s(\mathbf{x})$ denote a normalized image and let $\widetilde{s}(t,\theta)$ and $\widehat{s}(t,\theta)$ are the Radon transform and the R-CDT transform of $s(x)$, respectively. The R-CDT of $s_{g^\theta}=\mathscr{R}^{-1}\left(\left({g^\theta}\right)^\prime\widetilde{s}\circ {g^\theta}\right)$ is given by 
\begin{align}
    \widehat{s}_{g^\theta}=\left(g^\theta\right)^{-1}\circ\widehat{s}\nonumber
\end{align}

\begin{proof}
Let $r$ denote a reference image. Let $\widetilde{s}$ and $\widetilde{s}_{g^\theta}$ denote the Radon transforms of $s$ and $s_{g^\theta}$, respectively, and let $\widehat{s}$ and $\widehat{s}_{g^\theta}$ denote the CDTs of $s$ and $s_{g^\theta}$, respectively, with respect to the reference $r$. Then $\forall\theta\in[0,\pi]$, we have that
\begin{align}
    \int_{-\infty}^{\widehat{s}(t,\theta)}\widetilde{s}(u,\theta)du=\int_{-\infty}^{\widehat{s}_{g^\theta}(t,\theta)}\widetilde{s}_{g^\theta}(u,\theta)du=\int_{-\infty}^{t}\widetilde{r}(u,\theta)du\nonumber
\end{align}
If we substitute $s_{g^\theta}=\mathscr{R}^{-1}\left(\left({g^\theta}\right)^\prime\widetilde{s}\circ {g^\theta}\right)$ or, $\widetilde{s}_{g^\theta}=\left({g^\theta}\right)^\prime\widetilde{s}\circ {g^\theta}$. Then $\forall\theta\in[0,\pi]$, we have
\begin{align}
    \label{appbeq2}\int_{-\infty}^{\widehat{s}(t,\theta)}\widetilde{s}(u,\theta)du=\int_{-\infty}^{\widehat{s}_{g^\theta}(t,\theta)}\left(g^\theta\right)^\prime(u)\widetilde{s}\left(g^\theta(u),\theta\right)du
\end{align}
By the change of variables theorem, we can replace $g^\theta(u)=v$, $\left(g^\theta\right)^\prime(u)du=dv$ in equation~\eqref{appbeq2}:
\begin{align}
    \label{appbeq3}\int_{-\infty}^{\widehat{s}(t,\theta)}\widetilde{s}(u,\theta)du=\int_{-\infty}^{g^\theta\left(\widehat{s}_{g^\theta}(t,\theta)\right)}\widetilde{s}(v,\theta)dv,~\forall\theta\in[0,\pi]
\end{align}
From equation~\eqref{appbeq3}, we have that
\begin{align}
    &g^\theta\left(\widehat{s}_{g^\theta}(t,\theta)\right)=\widehat{s}(t,\theta)\nonumber\\
    \implies&\widehat{s}_{g^\theta}(t,\theta)=\left(g^\theta\right)^{-1}\left(\widehat{s}(t,\theta)\right)~\mbox{or,}~\widehat{s}_{g^\theta}=\left(g^\theta\right)^{-1}\circ\widehat{s} \nonumber
\end{align}
\end{proof}

\section{Proof of Property II-B.2}\label{appendix: cdtembedding}
Recall that given two signals $s$ and $r$, the Wasserstein metric $W_2(\cdot,\cdot)$ between them is defined in the following way:
\begin{equation}
	W_2^2(s,r) = \int_{\Omega_r} (\widehat s(x)-x)^2 r(x)dx,
\end{equation}
where $\widehat s$ is the CDT of $s$ with respect to $r$.
\begin{proof}
	Recall that an isometric embedding between two metric spaces is an injective mapping that preserve distances. Define the embedding by the correspondence $s \mapsto \widehat s$, it is left to show that 
\begin{equation*}
	 W_2^2(s_1,s_2)=\left|\left|\left(\widehat{s}_1-\widehat{s}_2\right)\sqrt{r}\right|\right|_{L^2(\Omega_r)}^2,
\end{equation*}
for all signals $s_1,s_2$. Let $f(y)$ be the CDT of $s_2$ with respect to $s_1$, then
\begin{equation*}
	W_2^2(s_2,s_1) = \int_{\Omega_{s_1}} (f(y)-y)^2 s_1(y)dy.
\end{equation*}
By the definition of CDT,  $s_1 = f^{\prime} s_2\circ f$ and  $r= \widehat s_1^{\prime}s_1\circ \widehat s_1 $. Then by the composition property, $\widehat s_1 = f^{-1}\circ \widehat s_2$. Here again $\widehat s_1, \widehat s_2$ are CDT with respect to a fixed reference $r$.
Let $y = \widehat s_1 (x)$.  Using the change of variables formula,
\begin{align*}
	 W_2^2(s_1,s_2)& = \int_{\Omega_r} (f(\widehat s_1(x)- \widehat s_1(x)) s_1(\widehat s_1(x)) \widehat s_1^{\prime}(x)dx\\
	&=\int_{\Omega_r} (\widehat s_2(x)- \widehat s_1(x))^2r(x)dx\\
	& = ||(\widehat s_2- \widehat s_1)\sqrt{r}||^2_{L^2(\Omega_r)}.
\end{align*}
\end{proof}

\section{Proof of Property II-D.2}
Recall that given two images $s,r$, using the correspondence in equation~\eqref{eq:rcdt} the Sliced Wasserstein metric $SW_2(\cdot,\cdot)$ is defined as follows:
\begin{equation}\label{eq: swmetric}
    SW_2^2(s,r) = \int_{\Omega_{\tilde r}} (\widehat s(t,\theta)-t)^2 \tilde r(t,\theta) dt d\theta.
\end{equation}
It can be shown that the above metric is well-defined\cite{kolouri2016radon}, and in particular
\begin{equation}\label{eq: sw}
 	SW_2^2(s_1,s_2)= \int_{\Omega_{\tilde r}} (\widehat s_1(t,\theta)-\widehat s_2(t,\theta))^2 \widetilde r(t,\theta) dt d\theta,
\end{equation}
for all images $s_1,s_2$, the proof of which is essentially the same as in the CDT case in Appendix \ref{appendix: cdtembedding}.

\begin{proof}
Recall that an isometric embedding between two metric spaces is an injective mapping that preserve distances. Define the embedding by $s(\bf{x}) \mapsto \widehat s(t,\theta)$ and the conclusion follows immediately from \eqref{eq: sw}. 
\end{proof}

\section{Proof of Lemma~\ref{lemma:nonoverlap}.}
Let $\S^{(k)}, k= 1,2,...,$ be the generative classes with a common confound set $\GO$ such that any $f\notin \GO$, $f^{\prime}\varphi^{(k)}\circ f\notin \S^{(k)}$.\footnote{This condition is automatically satisfied if $\varphi^{(k)}>0$ on $\R$ but may not hold in general if $\varphi^{(k)}$ is supported on a finite interval.}
\subsection*{Proposition: $\widehat{\S}^{(k)}\cap\widehat{\V}^{(p)}=\varnothing,~~\forall~k\neq p$.}
\subsection*{Assumptions:}
\begin{enumerate}
    \item $\S^{(k)}\cap\S^{(p)}=\varnothing$.
    \item $\left\{f(x)=ax|a>0\right\}\subseteq\GO$.
    \item $\GO$ is a convex group.
    \item $\forall~ \textrm{increasing function}~h\notin\GO$ and $0<\alpha<1$, $~\alpha ~id+(1-\alpha)h\notin\GO$ ($id$ denotes the identity function, $f(x)=x$).
\end{enumerate}

\begin{proof}
Before we prove the main claim, let us start by stating and proving the following claim:\\\\
{\emph{Claim (1)}}: $\forall~\widehat{s}_i^{(k)}\in\widehat{\S}^{(k)}$ and $\widehat{s}_j^{(p)}\in\widehat{\S}^{(p)}$ and $0<\alpha<1$, 
\begin{align}
\alpha \widehat{s}_i^{(k)}+(1-\alpha)\widehat{s}_j^{(p)}\notin \widehat{\S}^{(k)}\cup\widehat{\S}^{(p)}.\nonumber    
\end{align}\\
{\emph{Proof of Claim (1)}}: Let us prove by contradiction and assume that the claim is not true. Then, given $\alpha\in (0,1)$
\begin{align}
    &\alpha \widehat{s}_i^{(k)}+(1-\alpha)\widehat{s}_j^{(p)}\in \widehat{\S}^{(k)}.\nonumber\\
    \label{esw}\implies&\alpha\widehat{s}_i^{(k)}+(1-\alpha)\widehat{s}_j^{(p)}=g^{-1}\circ\widehat{\varphi}^{(k)},
\end{align}
for some $g\in \GO$.

Then, $\exists~h\notin\GO$, where $h\circ\widehat{s}_i^{(k)}=\widehat{s}_j^{(p)}$. Using this fact in equation~\eqref{esw} we have that,
\begin{align}
    &\alpha \widehat{s}_i^{(k)}+(1-\alpha)h\circ\widehat{s}_i^{(k)}=g^{-1}\circ\widehat{\varphi}^{(k)}\nonumber\\
    \implies& \left(\alpha~id+(1-\alpha)h\right)\circ g_i^{-1}\circ\widehat{\varphi}^{(k)}=g^{-1}\circ\widehat{\varphi}^{(k)};~g_i\in\GO\nonumber\\
    \label{esdws}\implies& f^{-1}\circ\widehat{\varphi}^{(k)}=g^{-1}\circ\widehat{\varphi}^{(k)}
\end{align}
where $f^{-1}=\left(\alpha~id+(1-\alpha)h\right)\circ g_i^{-1}$. Note that by assumption (4), $\alpha~id+(1-\alpha)h\notin\GO$. Since $g_i\in \GO$ and $\GO$ is a group, it follows that $f^{-1}\notin \GO$ and hence $f\notin \GO$. By the  assumption that  for any $f\notin \GO$, $f^{\prime}\varphi^{(k)}\circ f\notin \S^{(k)}$ (or equivalently $f^{-1}\circ \widehat \varphi^{(k)}\notin \widehat \S^{(k)}$), it follows that the LHS of \eqref{esdws} does not belong to $\S^{(k)}$, which is a contradiction since the RHS of \eqref{esdws} belongs to $\S^{(k)}$.
Therefore, 
\begin{align}
    &\alpha \widehat{s}_i^{(k)}+(1-\alpha)\widehat{s}_j^{(p)}\notin \widehat{\S}^{(k)}.\nonumber
\end{align}
Similarly, we can show that
\begin{align}
    &\alpha \widehat{s}_i^{(k)}+(1-\alpha)\widehat{s}_j^{(p)}\notin \widehat{\S}^{(p)}.\nonumber
\end{align}
In other words, 
\begin{align}
    &\alpha \widehat{s}_i^{(k)}+(1-\alpha)\widehat{s}_j^{(p)}\notin \widehat{\S}^{(k)}\cup\widehat{\S}^{(p)}.\nonumber
\end{align}
Therefore, Claim (1) is true.\\\\
{\emph{Main claim}}: 
\begin{align}
\widehat{\S}^{(k)}\cap\widehat{\V}^{(p)}=\varnothing,~~\forall~k\neq p\nonumber
\end{align}\\
{\emph{Proof of the main claim}}: Let us prove by contradiction and assume that the main claim is not true. Then, $\exists~\beta_j\in\mathbb{R}$ for some $g\in\GO$ such that 
\begin{align}
\label{dded}\sum_{j\in J}\beta_j\widehat{s}_j^{(p)}=g^{-1}\circ\widehat{\varphi}^{(k)}    
\end{align}
Let us consider the case when  $\beta_j>0$ for all $j\in J$.
Note that  the LHS of equation~\eqref{dded} is a member of $\widehat{\S}^{(p)}$. To see this, we note that by assumption (2) and Lemma \ref{lemma:convexity}, any convex combination of elements in $\widehat{\S}^{(p)}$ lies in $\widehat \S^{(p)}$, i.e., $ \sum\limits_{j\in J} \frac{\beta_j}{\sum\limits_{j\in J} \beta_j}\widehat{s}_j^{(p)}\in \widehat{\S}^{(p)}$. By assumption (3) and the composition property of the CDT (see Section~\ref{prelim_b}), we have that $\alpha^{-1}\circ \widehat s^{(p)}\in \widehat{\S}^{(p)}$ for any $\alpha>0$ and $\widehat s^{(p)}\in \widehat\S^{(p)}$. Letting $\alpha = (\sum\limits_{j\in J} \beta_j)^{-1}$ and $\widehat s^{(p)} = \frac{1}{\sum\limits_{j\in J} \beta_j} \sum\limits_{j\in J} \beta_j\widehat{s}_j^{(p)}$, we have that $\sum\limits_{j\in J} \beta_j\widehat{s}_j^{(p)}\in \widehat{\S}^{(p)}$.    Since the RHS of equation~\eqref{dded} lies in $\widehat{\S}^{(k)}$, it follows that  equation~\eqref{dded} cannot hold when  $\beta_j>0$ for all $j\in J$ as $\widehat{\S}^{(p)}\cap \widehat{\S}^{(k)} = \varnothing$ (by assumption (1) and Remark~\ref{lemma:nooverlap}). On the other hand, equation~\eqref{dded} cannot hold when $\beta_j<0$ for all $j\in J$ since the LHS of \eqref{dded}  would be a strictly decreasing function while the RHS is a strictly increasing function. Now, let us define the following:
\begin{align}
    J_+=\left\{j\in J|\beta_j>0\right\};~~J_-=\left\{j\in J|\beta_j<0\right\}\nonumber
\end{align}
Equation~\eqref{dded} then can be written as
\begin{align}
    &\frac{1}{2}\sum_{j\in J_+}\beta_j\widehat{s}_j^{(p)}+\frac{1}{2}\sum_{j\in J_-}\beta_j\widehat{s}_j^{(p)}=\frac{1}{2}g^{-1}\circ\widehat{\varphi}^{(k)}\nonumber \\
    \label{lsteq}&\frac{1}{2}\sum_{j\in J_+}\beta_j\widehat{s}_j^{(p)}=\frac{1}{2}\sum_{j\in J_-}\left(-\beta_j\right)\widehat{s}_j^{(p)}+\frac{1}{2}g^{-1}\circ\widehat{\varphi}^{(k)}  
\end{align}
Now as $\beta_j|_{j\in J_+}>0$ and $\left(-\beta_j\right)|_{j\in J_-}>0$, by assumption (2), $\sum_{j\in J_+}\beta_j\widehat{s}_j^{(p)}\in\widehat{S}^{(p)}$ and $\sum_{j\in J_-}\left(-\beta_j\right)\widehat{s}_j^{(p)}\in\widehat{S}^{(p)}$. Also, $g^{-1}\circ\widehat{\varphi}^{(k)}\in\widehat{\S}^{(k)}$. Now,
\begin{align}
    &\text{LHS of equation \eqref{lsteq}}\nonumber\\
    &=\frac{1}{2}\sum_{j\in J_+}\beta_j\widehat{s}_j^{(p)}\in\widehat{\S}^{(p)}\nonumber\\
    &\text{RHS of equation \eqref{lsteq}}\nonumber\\
    &=\frac{1}{2}\sum_{j\in J_-}\left(-\beta_j\right)\widehat{s}_j^{(p)}+\left(1-\frac{1}{2}\right)g^{-1}\circ\widehat{\varphi}^{(k)} \notin\widehat{\S}^{(k)}\cup\widehat{\S}^{(p)}\nonumber\\
    &~~~~~~~~~~~~~~~~~~~~~~~~~~~~~~~~~~~~~~~~~~~~~~~\text{(by using Claim (1))}\nonumber
\end{align}
which is a contradiction. Therefore, there exists no $\beta_j\in\mathbb{R}$ such that
\begin{align}
\sum_{j\in J}\beta_j\widehat{s}_j^{(p)}= g^{-1}\circ\widehat{\varphi}^{(k)}  \nonumber
\end{align}
which implies, the main claim is true, i.e., $\widehat{\S}^{(k)}\cap\widehat{\V}^{(p)}=\varnothing,~~\forall~k\neq p$. Note that, $\widehat{\S}^{(k)}$ here does not contain the origin because the generative models in equations~\eqref{eq:1dgenerative_model} and \eqref{eq:2dgenerative_model} do not allow for zero elements.
\end{proof}

\clearpage
\newpage
\onecolumn
\setcounter{table}{0} \renewcommand{\thetable}{A.\arabic{table}}

\section{Standard deviation of test accuracy}

\begin{table*}[!hbt]
\centering
\caption{Standard deviation of percentage test accuracy in the Chinese printed character dataset.}
\begin{tabular}{cccccc}
\hline
            & \multicolumn{5}{c}{No. of training samples (per class)} \\ \cline{2-6} 
            & 1         & 2        & 4        & 8         & 16        \\ \hline
Resnet      & 0.08       & 0.21     & 2.45     & 4.34      & 0.17      \\
Shallow-CNN & 0.04      & 0.06     & 0.17     & 0.82      & 2.21      \\
VGGnet      & 0         & 0        & 0.87     & 20.32     & 41.54     \\
Proposed    & 0.21      & 0.28     & 0.04     & 0         & 0         \\ \hline
\end{tabular}
\end{table*}

\begin{table*}[!hbt]
\centering
\caption{Standard deviation of percentage test accuracy in the MNIST dataset.}
\begin{tabular}{cccccccccccccc}
\hline
            & \multicolumn{12}{c}{No. of training samples (per class)} \\ \cline{2-14} 
            & 1         & 2        & 4        & 8         & 16    & 32 & 64 & 128 & 256 & 512 & 1024 & 2048 &4096    \\ \hline
Resnet      & 3.29&    4.05&    3.03 &    9.13 &    8.04 &    2.01 &    1.85&    0.95&    0.45 & 0.75  &  0.12 &    0.15 &    0.06     \\
Shallow-CNN & 4.08&    6.89&    2.64&    1.12&    3.90&    0.96&    1.42&    0.49&    0.31&0.23 &   0.09&    0.09&    0.07      \\
Proposed    & 5.25&    7.97&    4.27&    1.21&    1.48&    0.50&    0.33&    0.20&    0.16&0.08&    0.11&    0.08&    0.07         \\ \hline
\end{tabular}
\end{table*}

\begin{table*}[!hbt]
\centering
\caption{Standard deviation of percentage test accuracy in the Affine-MNIST dataset.}
\begin{tabular}{cccccccccccccc}
\hline
            & \multicolumn{12}{c}{No. of training samples (per class)} \\ \cline{2-14} 
            & 1         & 2        & 4        & 8         & 16    & 32 & 64 & 128 & 256 & 512 & 1024 & 2048 &4096    \\ \hline
Resnet      & 1.45 &         1.08&          0.64    &      1.08 &         6.48 &3.86          &3.95&          1.56&          0.27&          0.21&0.16&          0.21&          0.09    \\
Shallow-CNN & 1.18   &       1.31 &         1.09   &       0.67     &     2.58  & 2.06      &    2.95    &      1.65  &        1.03  &        0.38 &0.45    &      0.33      &    0.27   \\
VGGnet & 2.59    &      2.99    &      3.17     &     4.67      &    4.78 &   2.97 &         1.35 &         0.99        &  0.45 &         0.33 & 0.18 &         0.17    &      0.12    \\
Proposed    & 3.27  &        5.29   &       2.31    &      2.30 &         1.33 &   0.59 &         0.38 &         0.22 &         0.15    &      0.1 &  0.08 &         0.08 &         0.08     \\ \hline
\end{tabular}
\end{table*}

\begin{table*}[!hbt]
\centering
\caption{Standard deviation of percentage test accuracy in the Optical OAM dataset.}
\begin{tabular}{cccccccccccccc}
\hline
            & \multicolumn{10}{c}{No. of training samples (per class)} \\ \cline{2-11} 
            & 1         & 2        & 4        & 8         & 16    & 32 & 64 & 128 & 256 & 512     \\ \hline
Resnet      & 1.71 &         4.31 &         2.60    &      1.39 &         0.84 & 0.78  &        0.22    &      0.05 &         0.04 &         0.16    \\
Shallow-CNN & 2.80 &         1.03 &         2.64    &      1.72 &         4.29 & 0.81  &        0.45    &      0.10 &         0.18 &         0.12   \\
VGGnet & 1.64     &     1.63    &     13.30 &        13.11 &         2.81 & 1.97    &      0.77     &     0.44  &        0.12 &         0.05    \\
Proposed    & 2.40 &         1.73 &         0.66    &      0.54 &         0.28 & 0.09   &       0.02    &      0.01     &     0.01      &    0.01     \\ \hline
\end{tabular}
\end{table*}

\begin{table*}[!hbt]
\centering
\caption{Standard deviation of percentage test accuracy in the Sign language dataset.}
\begin{tabular}{cccccccccccccc}
\hline
            & \multicolumn{10}{c}{No. of training samples (per class)} \\ \cline{2-11} 
            & 1         & 2        & 4        & 8         & 16    & 32 & 64 & 128 & 256 & 512     \\ \hline
Resnet      & 9.08  &       15.14 &         9.24    &     12.26     &    11.80 & 7.02   &       4.94    &      2.03 &         0.76 &         0.05    \\
Shallow-CNN & 9.87 &         4.49  &        3.07    &      1.62     &     5.93 &  7.58      &    1.62 &         1.22    &      0.03 &            0   \\
VGGnet & 8.83   &      15.48    &     16.35     &    19.79  &        1.76 &  5.67 &         3.76        &  1.22     &     1.39 &         0.27     \\
Proposed    & 12.26 &         9.68  &        6.85  &        4.18        &  1.73  &  0.78        &  0.12 &            0  &           0   &          0    \\ \hline
\end{tabular}
\end{table*}

\begin{table*}[!hbt]
\centering
\caption{Standard deviation of percentage test accuracy in the OASIS brain MRI dataset.}
\begin{tabular}{ccccccc}
\hline
            & \multicolumn{5}{c}{No. of training samples (per class)} \\ \cline{2-7} 
            & 1         & 2        & 4        & 8         & 16  & 32      \\ \hline
Resnet      & 4.40  &       11.58  &       11.69    &     12.09     &    12.51 &         7.96   \\
Shallow-CNN & 18.12     &    17.42      &    8.28 &         5.49        & 12.68  &        6.37     \\
VGGnet      & 5.07 &         5.12   &       4.06    &      4.50     &    11.99  &       10.02    \\
Proposed    & 7.56 &        5.43  &        3.56 &        2.96        &  2.26 &         0.85      \\ \hline
\end{tabular}
\end{table*}

\end{document}